\documentclass{article}

\usepackage[final]{neurips_2025}

\usepackage[utf8]{inputenc} 
\usepackage[T1]{fontenc}    
\usepackage{hyperref}       
\usepackage{url}            
\usepackage{booktabs}       
\usepackage{amsfonts}       
\usepackage{nicefrac}       
\usepackage{microtype}      
\usepackage{xcolor}         
\usepackage{natbib}
\usepackage{amsmath}
\usepackage{amssymb}
\usepackage{mathtools}
\usepackage{amsthm}
\usepackage{graphicx}
\newcommand{\cmark}{\ding{51}}%
\newcommand{\xmark}{\ding{55}}%
\newcommand{\ignore}[1]{}
\newcommand\C[1]\null
\hypersetup{
    colorlinks,
    linkcolor={blue!50!black},
    citecolor={blue!50!black},
    urlcolor={blue!80!black}
}

\usepackage[markup=underlined]{changes}

\usepackage[utf8]{inputenc}   
\usepackage[T1]{fontenc}    	
\usepackage{url}              		 
\usepackage{booktabs}       	
\usepackage{amsfonts}       	
\usepackage{nicefrac}       	  
\usepackage{microtype}      	

\usepackage{graphicx}
\usepackage{color}
\usepackage{rotating}
\usepackage{tabularx}
\usepackage{pdflscape}
\usepackage{amsmath,amssymb}
\usepackage{algorithm,algorithmic}
\usepackage{bbm,dsfont}
\usepackage{subfig}
\usepackage{caption}
\usepackage{wrapfig}
\usepackage{cancel}
\usepackage{enumerate,cases}
\usepackage{thmtools,thm-restate}
\usepackage{mathtools}
\usepackage{amssymb}

\usepackage[none]{hyphenat}
\usepackage{multirow}
\usepackage{makecell}
\usepackage{setspace}
\usepackage{pifont}
\usepackage{array}
\usepackage{enumitem}
\usepackage{lineno}     

\allowdisplaybreaks

\usepackage{hyperref}		 
\hypersetup{
	colorlinks	= true,		 
	urlcolor     = blue,	 
	linkcolor	 = purple, 	 
	citecolor    = violet    
}
\usepackage[capitalize]{cleveref}

\usepackage{todonotes}

\newtheorem{theorem}{Theorem}
\newtheorem{lemma}{Lemma}

\newtheorem{assump}{Assumption}

\renewcommand{\citet}[1]{\cite{#1}}

\title{On the Sample Complexity Bounds of \\ Bilevel Reinforcement Learning}

\author{{Mudit Gaur} \\
	Purdue University
	\And
    {Utsav Singh} \\
	IIT Kanpur
    \And
    {Amrit Singh Bedi\thanks{Equal contribution.}} \\
	University Of Central Florida
    \And
    {Raghu Pasupathy\footnotemark[1]} \\
    Purdue University\\
 	\And
    {Vaneet Aggarwal\footnotemark[1]}\\
	Purdue University    
}

\begin{document}
\maketitle

\begin{abstract} \label{abstract}
Bilevel reinforcement learning (BRL) has emerged as a powerful framework for aligning generative models, yet its theoretical foundations, especially sample complexity bounds, remain relatively underexplored. In this work, we present the first sample complexity bound for BRL, establishing a rate of $\tilde{\mathcal{O}}(\epsilon^{-3})$ in continuous state-action spaces. Traditional MDP analysis techniques do not extend to BRL due to its nested structure and non-convex lower-level problems. We overcome these challenges by leveraging the Polyak-Łojasiewicz (PL) condition and the MDP structure to obtain closed-form gradients, enabling tight sample complexity analysis. Our analysis also extends to general bi-level optimization settings with non-convex lower levels, where we achieve state-of-the-art sample complexity results of $\tilde{\mathcal{O}}(\epsilon^{-3})$ improving upon existing bounds of  $\tilde{\mathcal{O}}(\epsilon^{-6})$. Additionally, we address the computational bottleneck of hypergradient estimation by proposing a fully first-order, Hessian-free algorithm suitable for large-scale problems.
\end{abstract}

\section{Introduction} \label{introduction}

Bilevel reinforcement learning (BRL) has emerged as a powerful framework for modeling hierarchical decision-making processes, particularly in the context of artificial intelligence (AI) alignment. Recent works, such as those by \citet{chakraborty2024parl,ding2024sail,shen2023penaltybasedbilevelgradientdescent,srivastava2025technical}, have demonstrated the potential of bilevel formulations to address challenges in reinforcement learning from human feedback (RLHF) and inverse reinforcement learning. Despite these advancements, the theoretical understanding of BRL remains limited, especially concerning sample complexity in parameterized settings. Most existing theoretical analyses such as \citep{yang2024bilevel} are confined to tabular settings due to their analytical tractability, while empirical studies \citep{fu2017learning} are conducted in parameterized environments, leading to a \textit{disconnect between theory and practice}.

A significant challenge in bridging this gap is the computational complexity associated with the hierarchical structure of BRL. Practical algorithms often circumvent the need for second-order gradient evaluations to solve bilevel problems by employing first-order approximations \citep{chakraborty2024parl}. However, these simplification raises critical questions about the theoretical performance loss incurred by such approximations in BRL. Therefore, understanding and deriving tight sample complexity bounds for BRL are crucial for guiding the development of more efficient algorithms and for assessing the trade-offs between theoretical rigor and practical applicability. To this end, we present the first sample complexity result for BRL in continuous state action settings, achieving a bound of $\tilde{\mathcal{O}}(\epsilon^{-3})$. This result extends to standard bilevel optimization problems, providing a significant theoretical contribution wixth practical implications.

\textbf{Key challenges and our approach.} The theoretical analysis of BRL is not possible using the existing theoretical frameworks \cite{gaur2024closinggapachievingglobal,mondal2024improved,ganesh2025order,ganesh2025order2,ganesh2025a} used to analyze MDP algorithms with a known reward function.  Existing bi-level algorithms are also ill-suited to the BRL setup since they require unbiased gradients \cite{pmlr-v151-chen22e,grazzi2023bilevel}, which are not available in the BRL setup. Many bi-level algorithms \cite{chen2023decentralized,ji2021bilevel} also require the estimation of second-order terms such as Hessian, which make them computationally infeasible as well in high-dimensional setups. Some works in the field of BRL do employ the  approximation of the second-order Hessian \cite{chakraborty2024parl,yang2024bilevel}. However, these works are limited to tabular state spaces. Other approaches such as  \cite{shen2023penaltybasedbilevelgradientdescent} use a penalty based reformulation of the BRL problem. This work is still restricted to the tabular setup. From a theoretical standpoint, none of the works develop a method to analyze the sample complexity of the work for continuous state action space, 
works that have a sample complexity such as \cite{yang2024bilevel} only do so for a tabular state space. We overcome this challenge by (i) proposing a first-order BRL algorithm that works for continuous state-action spaces, (ii) providing the first-ever sample complexity results for a BRL algorithm. We use a penalized bi-level framework with non-convex lower level initially proposed in \cite{kwon2024penaltymethodsnonconvexbilevel} for standard optimization, but it is not straightforward to apply to reinforcement learning settings, which is the main focus of our work. 

In order to obtain our sample complexity result, we use the insight that the gradient parameter estimation step  in the algorithm laid out in \cite{chen2024findingsmallhypergradientsbilevel} (lines 3-8 of Algorithm \ref{algo_1}) are an SGD step on a loss function that satisfies the  Polyak-Łojasiewicz (PL) property. We combine this insight with our novel recursive analysis of the optimality gap (lemma \ref{thm_supp}) for stochastic gradient descent (SGD) with biased gradient estimate to obtain the first ever sample complexity result for BRL. We also demonstrate that our analysis holds for the standard bi-level penalty-based formulation of \cite{kwon2024penaltymethodsnonconvexbilevel} with unbiased gradient estimates and provides state-of-the-art sample complexity results for the same (Theorem \ref{thm}).

We obtain this result with our unique insight {that the parameter estimation steps of the algorithm provided in \citet{chen2024findingsmallhypergradientsbilevel}, are gradient descent steps on a loss function that satisfies the Polyak-Łojasiewicz (PL) property}. Additionally, we use the MDP structure to obtain closed form expressions of the gradients. This allows us to obtain the first ever sample complexity of BRL. Our analysis also applies to the general bi-level setup as a special case where we obtain state of the art sample complexity results of ${\epsilon}^{-4}$.

We summarize our main contributions as follows.

\begin{itemize}[leftmargin=*]
    \item \textbf{Novel sample complexity bounds in BRL:} We derive the first sample complexity bounds for BRL with parameterized settings, achieving a bound of $\tilde{\mathcal{O}}(\epsilon^{-3})$. Our analysis addresses the challenges posed by non-convex lower-level problems and does not rely on computationally expensive second-order derivatives. 

    \item \textbf{Generalization to standard bilevel optimization:}
    Our theoretical results extend beyond reinforcement learning to standard bilevel optimization problems, assuming access to unbiased gradients for the upper and lower level objectives. For setups with non-convex lower-level problems, our method achieves a state-of-the-art sample complexity of $\tilde{\mathcal{O}}(\epsilon^{-3})$.

    \item We perform proof of concept experiments on Mujoco Tasks to demonstrate that the proposed first order algorithm  works in practice.
\end{itemize}

\section{Related Works} \label{related_works}



\begin{table}[t]\label{1}
\centering
 \caption{This table shows a comparison of state-of-the-art sample complexity results for bilevel reinforcement learning (BRL). Our result is among the first to establish sample complexity bounds for continuous state-action spaces.\vspace{2mm}}
\label{tbl_related_1}%
{ 
\begin{tabular}{|c|c|c|c|c|c|c|}
\hline
References & \multicolumn{1}{c|}{\begin{tabular}[c]{@{}c@{}}  Continuous  \\ Space \end{tabular}} & \multicolumn{1}{c|}{\begin{tabular}[c]{@{}c@{}} Iteration \\Complexity \end{tabular}} &  \multicolumn{1}{c|}{\begin{tabular}[c]{@{}c@{}} Sample \\ Complexity \end{tabular}} \\ \hline

        \cite{shen2024principled} &    \textcolor{red}{\xmark}              & 
 $\tilde{\mathcal{O}}(\epsilon^{-1})$ &   \textcolor{red}{\xmark}             \\ 
        \cite{chakraborty2024parl} &    \textcolor{red}{\xmark}              & 
 $\tilde{\mathcal{O}}(\epsilon^{-1})$ &   \textcolor{red}{\xmark} \\ 
\cite{yang2024bilevel} &    \textcolor{red}{\xmark}              & 
 $\tilde{\mathcal{O}}(\epsilon^{-1.5})$ &  $\tilde{\mathcal{O}}(\epsilon^{-3.5})$  \\ 
 \cite{shen2023penaltybasedbilevelgradientdescent} &    \textcolor{red}{\xmark}              & 
 $\tilde{\mathcal{O}}(\epsilon^{-1})$ &   \textcolor{red}{\xmark} \\ 
        This Work &    \textcolor{green}{\cmark}              & 
 $\tilde{\mathcal{O}}(\epsilon^{-1})$ &   $\tilde{\mathcal{O}}(\epsilon^{-3})$             \\    
    \hline
\end{tabular}
}
\end{table}
We first go over the prevailing literature in the field of bilevel optimization. Once we have established a broad overview of the existing results in the field, we will lay out the existing results in the field of BRL and how they compare to the bilevel optimization results. 

\textbf{Bilevel optimization } problems have been studied extensively from the theoretical perspective in recent years. Approaches such as \citet{ji2021bilevel} have been shown to achieve convergence, but with expensive evaluations of Hessian / Jacobian matrices and Hessian / Jacobian vector products. Works such as \citet{sow2022convergence,yang2023achieving} forgo the use of exact Hessian/Jacobian matrices but instead approximate them. Works such as \citet{stoch_bil2} do not require even the approximation of the second-order terms. However, in all of the aforementioned works, the lower level is restricted to be convex. In general, bilevel optimization with non-convex lower-level objectives is not computationally tractable without further assumptions, even for the special case of min-max optimization \citep{daskalakis2020complexityconstrainedminmaxoptimization}. Therefore, additional assumptions are necessary for the lower-level problem. The work in \cite{kwon2024penaltymethodsnonconvexbilevel} established a penalty-based framework for solving bilevel optimizations with a possible non-convex lower levels with the PL assumption on the lower-level function. The work in \cite{chen2024findingsmallhypergradientsbilevel} obtained convergence in the bilevel setup with a non-convex lower level with an improved sample complexity with respect to \cite{kwon2024penaltymethodsnonconvexbilevel}, where it obtained $\epsilon^{-6}$ compared to $\epsilon^{-7}$. Note that PL assumption has been used previously in fields such as optimization \citep{karimi2016linear}, RL \citep{gaur2023global,pmlr-v202-gaur23a} as well as analysis of modern ML methods such as flow matching \citep{gaur2025generativemodelingcontinuousflows} and diffusion models \citep{gaur2025generativemodelingcontinuousflows,srikanth2025discretestatediffusionmodels}.



\textbf{Bilevel reinforcement learning} has been used in several applications such as RLHF
\citep{christiano2017deep,xu2020preference},
reward shaping  \citep{Zou_Ren_Yan_Su_Zhu_2021}, Stackelberg Markov
game \citep{liu2021sharp,song2023can}, AI-economics with two-level deep RL \citep{zheng2022ai}, social environment design \citep{zhang2024social}, incentive design \citep{chen2016caching}, etc. Another recent work \citep{chakraborty2024parl}  studies the policy alignment problem and introduces
a corrected reward learning objective for RLHF that leads to strong performance gain. There are a very limited number of theoretical convergence results for such a setup. The PARL algorithm \citep{chakraborty2024parl} achieves convergence of the BRL setup using the implicit gradient method that requires not only the strong convexity of the lower-level objective but also necessitates the use of second-order derivatives. Note that in general the lower level of BRL is the discounted reward which is not convex. The work of \citet{shen2024principled} employs a penalty-based framework to achieve convergence for a BRL setup using a first-order algorithm. Similarly, \citet{yang2024bilevel} establishes convergence by deriving an expression for the hypergradient without assuming convexity of the lower-level problem. However, it is important to note that all existing convergence results in BRL thus far do not provide sample complexity for continuous state action spaces. Despite the existence of sample complexity results for bilevel optimization with non-convex lower-level objectives in the broader bilevel literature without tabular state space restriction, such results remain absent in the context of BRL. 


\section{Problem Formulation} \label{Problem Setup}

\noindent \textbf{Markov Decision Process (MDP).} We consider a discounted MDP defined by the tuple ${\mathcal{M}} = (\mathcal{S}, \mathcal{A}, P, r_{\phi}, \gamma)$, where $\mathcal{S}$ is a bounded measurable state space and $\mathcal{A}$ is a bounded measurable action space. We remark that in our setup, both the state and action spaces can be infinite, though they remain bounded. In the MDP, ${P} : \mathcal{S} \times \mathcal{A} \rightarrow \mathcal{P}(\mathcal{S})$ is the probability transition function and $r_{\phi} : \mathcal{S} \times \mathcal{A} \rightarrow [0,1]$ represents the parameterized reward function, ($\phi \in \Theta$) where $\Theta$ is a compact space. 
In order to encourage exploration, in many cases an additional KL-regularization term is preferred. This can be accounted for by defining the reward function as 
\begin{align}
         r_{\phi}(s,a) =  r_{\phi}(s,a) + {\beta}h_{\pi,\pi_{\text{ref}}}(s,a), 
\end{align}
where $h_{\pi,\pi_{\text{ref}}}(s_{i},a_{i})= {\log}\left(\frac{\pi(a_{i}|s_{i})}{\pi_{\text{ref}}(a_{i}|s_{i})}\right)$ is the KL regularization term where $\pi_{\text{ref}}$  is the reference policy. This form of the KL penalty is used in RLHF works such as in \citet{ziegler2019fine}. Note that our analysis works for any regularization term that is uniformly bounded.
Finally, $0 < \gamma < 1$ is the discount factor. 
A policy $\pi : \mathcal{S} \rightarrow \mathcal{P}(\mathcal{A})$ maps each state to a probability distribution over the action space. The state-action value function or $Q$ function is defined as follows:
\begin{align}
Q_{\phi}^{\pi}(s,a) = \mathbb{E}\left[\sum_{t=0}^{\infty}\gamma^{t}r_{\phi}(s_{t},a_{t})|s_{0}=s,a_{0}=a\right].   \label{mdp_ps_1} 
\end{align}
For a discounted MDP, we  define the optimal
action value functions as
\begin{align}
Q_{\phi}^{*}(s,a) =  \sup_{\pi}Q_{\phi}^{\pi}(s,a),  \hspace{0.5cm} \forall (s,a) \in  \mathcal{S}\times\mathcal{A} \label{mdp_ps_2}  .  
\end{align}
We have the expected average return given by 
 \begin{align}
     J(\phi,\lambda) = \mathbb{E}_{s \sim \nu, a \sim \pi_{\lambda}(.|s)}[Q_{\phi}^{\pi_{\lambda}}(s,a) ], \label{mdp_ps_3}
 \end{align}
where the policy is parameterized as $\{\pi_{\lambda} , \lambda \in \Lambda\}$ and $\Lambda$ is a compact set. 

\noindent \textbf{Bilevel reinforcement learning (BRL).} With the above notation in place, we can formulate the BRL problem as
\begin{align} 
    &\min_{\phi} G(\phi,\ \lambda^*(\phi)) \nonumber\\
    &\text{where } \ \lambda^*(\phi) = \arg\min_{\lambda} -J(\phi,\lambda), \label{mdp_ps_4}
\end{align}
where the upper-level objective $G(\phi,\lambda^*(\phi))$ is a function of the reward parameter $\phi$, while the lower-level objective is a function of the policy parameter $\lambda$. We denote the lower level loss function as $-J(\phi,\lambda)$ as opposed to $J(\phi,\lambda)$ 
to keep our notation in line with the bi-level literature; a similar notation is followed in \citet{shen2024principled}. 

\noindent \textbf{Existing approaches and limitations.} To solve the problem in \eqref{mdp_ps_4}, one popular approach is to rewrite the problem in \eqref{mdp_ps_4} in the following manner
\begin{align}
   & \min_{\phi}{\Phi}(\phi):=  G(\phi,\lambda^{*}(\phi)) \nonumber\\
&\text{where } \ \lambda^*(\phi) = \arg\min_{\lambda} -J(\phi,\lambda),\label{mdp_ps_6},
\end{align}
which is known as the \textit{hyper-objective} approach, where $\Phi$ is the hyper-objective. To solve it,  we need the calculation of the hyper-gradient given by
\begin{equation}
    {\nabla}_{\phi}\Phi(\phi) =  \nabla_{\phi}G(\phi,\lambda^{*}(\phi)) + v.\nabla_{\lambda}G(\phi,\lambda^{*}(\phi)), 
    \label{mdp_ps_7}
\end{equation}
where the term $v$ apart from the gradient of $\Phi$ is given as 
\begin{equation}
   v = -[{\nabla}^{2}_{\lambda,}J(\phi,\lambda^{*}(\phi))]^{-1}{\nabla}^{2}_{\phi,\lambda}J(\phi,\lambda^{*}(\phi))
    \label{mdp_ps_8} 
\end{equation}
    This approach has been used in the existing literature \citep{yang2023achieving,sow2022convergence,chakraborty2024parl}. Apart from having to calculate the Hessian and its inverse, this technique requires that the lower-level objective $J$ be convex. One solution, which is employed in \citet{yang2023achieving,sow2022convergence}, is to estimate first-order approximations of the Hessian. This is because the calculation of second-order terms, which in many cases can get prohibitively expensive from a computational perspective. 
%



\section{Proposed Approach} \label{Proposed Algorithm}
To avoid computationally expensive Hessians and for situations where the lower levels are not necessarily convex, penalty-based methods such as those developed in \citet{kwon2024penaltymethodsnonconvexbilevel} have been proposed. Based on that, in this paper, we consider the proxy objective 
\begin{align}
    {\Phi}_{\sigma}(\phi) =  \min_{\lambda} \left( G(\phi,\lambda)  +  \frac{J(\phi,{\lambda}^{*}(\phi))-J({{\phi}},\lambda)}{\sigma}\right)  \label{mdp_ps_9},
\end{align}
where $\sigma$ is a positive constant. The gradient of ${\Phi}_{\sigma}(\phi)$ is given by
\begin{align}
    {\nabla}_{\phi}{\Phi}_{\sigma}(\phi) =&   {\nabla}_{\phi}G(\phi,\lambda_{\sigma}^{*}(\phi)) + \frac{{\nabla}_{\phi}J({\phi},{\lambda^{*}(\phi)}) -  {\nabla}_{\phi}J(\phi,{\lambda_{\sigma}^{*}(\phi)})}{\sigma}, \label{mdp_ps_10}
\end{align} 
where $\lambda^{*}(\phi)=\arg\min_{\lambda} -J(\phi,\lambda)$ and  $\lambda_{\sigma}^{*}(\phi) =  \arg\min_{\lambda} -(J(\phi,\lambda) -  {\sigma}G(\phi,\lambda))$. For future notational convenience, we define the penalty function $h_{\sigma}(\phi,\lambda) =J(\phi,\lambda) -  {\sigma}G(\phi,\lambda)$.
A key advantage of this formulation is the fact that, unlike the method involving the hyper-gradient, it does not require the calculation of costly second-order terms. It is also applicable to setups where the lower level is non-convex. Despite these advantages, the theoretical analysis of this setup (even for the standard bi-level framework) is not well explored. 

\textbf{Remark (differences with \citep{kwon2024penaltymethodsnonconvexbilevel, chen2024findingsmallhypergradientsbilevel})}. Existing analyses in standard bilevel optimization settings have achieved sample complexities of $\tilde{\mathcal{O}}(\epsilon^{-7})$ and $\tilde{\mathcal{O}}(\epsilon^{-6})$ in \citet{kwon2024penaltymethodsnonconvexbilevel} and \citet{chen2024findingsmallhypergradientsbilevel}, respectively. These results apply to bilevel problems without an MDP structure, where the lower-level objective is non-convex but it is reasonable to assume access to unbiased gradient estimates with bounded variance for both upper- and lower-level objectives. However, such assumptions do not hold in bilevel reinforcement learning (BRL), where gradient estimates are inherently biased due to the underlying MDP dynamics. In this work, we develop a sample complexity analysis tailored to the BRL setting. We also specialize our analysis to the standard bilevel optimization setup and demonstrate that our approach yields improved sample complexity bounds compared to prior work (see Table 2).

\noindent \textbf{Algorithm development.} We will describe the algorithm to solve the problem described in Equation \eqref{mdp_ps_9}. We achieve this by implementing a gradient descent step in which the gradient is given by the expression in Equation \eqref{mdp_ps_10}. %
In order to estimate this gradient, we have to estimate the three terms ${\nabla}_{\phi}G(\phi,\lambda_{\sigma}^{*}(\phi))$, ${\nabla}_{\phi}J({\phi},\lambda^{*}(\phi))$ and ${\nabla}_{\phi}J(\phi,{\lambda_{\sigma}^{*}(\phi)})$. In turn, these terms require the estimation of the terms $\lambda^{*}(\phi)$ and $\lambda_{\sigma}^{*}(\phi)$.

For the gradient of $J(\phi,\lambda)$ with respect to the upper level variable and reward parameter $\phi$, note that there was no existing  closed-form expression. We show in Lemma \ref{lemma_5} in the Appendix \ref{lem_proof} that a closed form of ${\nabla}_{\phi}J(\phi,\lambda)$ is given by
\begin{align}
      \nabla_{\phi}J(\phi,\lambda)  =  \sum_{i=1}^{\infty} {\gamma}^{i-1} \mathbb{E}{\nabla}_{\phi}r_{\phi}(s_{i},a_{i}) \label{mdp_ps_12},
\end{align}

Here, the expectation is over the state action distribution induced by the policy $\lambda$.  This expression is obtained by following an argument similar to the proof of the policy gradient theorem in \citet{sutton1999policy}.
Note that we can only obtain a truncated estimate for $\nabla_{\phi}J({\phi},{\lambda})$, which will also lead to bias. In Algorithm \ref{algo_1}, we take an average of this truncated estimate over $B$ batches for a more stable estimate. We define the sample-based average here as

\begin{align}
     {\nabla}_{\phi}{J}(\phi,\lambda,B) = \frac{1}{B} \sum_{j=1}^{B}{\nabla}_{\phi}\hat{J}_{j}(\phi,\lambda). \label{bl_0}
\end{align}

where ${\nabla}_{\phi}\hat{J}_{j}(\phi,\lambda) = \sum_{j=1}^{H}{\nabla}_{\phi}r_{\phi}(s_{j,i},a_{j,i})$. Here, $(s_{j,i},a_{j,i})$ are the $i^{th}$ state-action pair of the $j^{th}$ trajectory sampled from the policy $\pi_{\lambda}$. 

For the gradient for the lower-level loss function gradient $J(\phi,\lambda)$ with respect to the lower-level variable $\lambda$ we use the policy gradient function to obtain
\begin{align}
   \nabla_{\lambda}J({\phi},{\lambda}) &=  \mathbb{E}_{(s,a)\sim d_{\nu}^{\pi_{\lambda}}}[{\nabla}_{\lambda}{\log}{\pi_{\lambda}}(a|s)Q^{\lambda}_{\phi}(s,a)]  \nonumber\\
   &+   \mathbb{E}_{(s_{i},a_{i} \sim \pi_{\lambda})}{\beta}\sum_{i=1}^{\infty}{\gamma}^{i-1}{\nabla}_{\lambda}h_{\pi_{\lambda},\pi_{ref}}(s_{i},a_{i}) \label{mdp_ps_11} 
\end{align}
Here $d_{\nu}^{\pi_{\lambda}}$ denotes the stationary distribution of the state action space induced by the policy $\pi_{\lambda}$. The second term on the right-hand side is due to the presence of the KL regularization term in the reward $r(\phi)$. Note that in real-world applications of RL algorithms, such as actor-critic, the estimate of $Q^{\lambda}_{\phi}$ is not an unbiased estimate, but instead a parametrized function, such as a neural network, is used to approximate it, leading to bias. Additionally we cannot sample the infinite sum $\mathbb{E}_{(s_{i},a_{i} \sim \pi_{\lambda})}{\beta}\sum_{i=1}^{\infty} {\nabla}_{\lambda}h_{\pi_{\lambda},\pi_{ref}}(s_{i},a_{i})$ but have to get a finite truncated estimate, which also leads to bias.We denote by $\nabla_{\lambda}\hat{J}({\phi},{\lambda},n,B)$ the estimate of $\nabla_{\lambda}{J}({\phi},{\lambda})$ as  

\begin{align}
   \nabla_{\lambda}\hat{J}({\phi},{\lambda},n,B) &=  \frac{1}{n}\sum_{i=1}^{n}[{\nabla}_{\lambda}{\log}{\pi_{\lambda}}(a_{i}|s_{i})\hat{Q}^{\lambda}_{\phi}(s_{i},a_{i})]  \nonumber\\
   &+     \frac{\beta}{B}\sum_{j=1}^{B}\sum_{i=1}^{H}{\gamma}^{i-1}{\nabla}_{\lambda}h_{\pi_{\lambda},\pi_{ref}}(s_{j,i},a_{j,i}) \label{mdp_ps_12_1}
\end{align}

Note that the estimate of ${Q}^{\lambda}_{\phi}(s,a)$ denoted by $\hat{Q}^{\lambda}_{\phi}(s,a)$ is estimated using $n$ samples.  For upper-level loss functions, unbiased gradient estimates can be calculated, as demonstrated in \citet{chakraborty2024parl}. For notational convenience, we define
\begin{align}
    {\nabla}G(\phi,\lambda,B) = \frac{1}{B} \sum_{i=1}^{B}{\nabla}\hat{G}_{i}(\phi,\lambda), \label{mdp_ps_13} 
\end{align}
where $B$ is the size of the gradient sample dataset and ${\nabla}\hat{G}_{i}(\phi,\lambda)$ is the gradient estimate sample $i^{th}$. Note here that the batch size $B$ and horizon length $H$ can vary across the different gradients. We keep this notation the same across gradients with respect to $\phi$ and $\lambda$ for notational convenience.   

Now that we have expressions for the gradients of the upper and lower level function, we now move onto the estimation of  ${\nabla}_{\phi}J({\phi},\lambda^{*}(\phi))$ and ${\nabla}_{\phi}J(\phi,{\lambda_{\sigma}^{*}(\phi)})$.
Consider the term ${\lambda_{\sigma}^{*}(\phi)}$ which is a minimizer of the function given by $h_{\sigma}(\phi,\lambda)$. Thus, it is obtained by performing a gradient descent on $h_{\sigma}(\phi,\lambda)$ with respect to $\lambda$. 
Similarly, ${\lambda^{*}(\phi)}$  is the minimizer of the function given by $J(\phi,\lambda)$ and can be obtained by gradient descent. Note that these steps are performed on lines 4-7 of Algorithm \ref{algo_1}.
The gradient descent step for the proxy loss function $\Phi_{\sigma}(\phi)$ is performed on line 11. We estimate the gradients of $G(\phi,\lambda)$ and $J(\phi,\lambda)$ with respect to $\phi$ using the expression in Equations $\eqref{mdp_ps_12}$ and  $\eqref{mdp_ps_13}$.

\begin{algorithm}[ht]
	\caption{A first-order approach to bilevel RL}
	\label{algo_1}
    \begin{algorithmic}[1]
    \STATE 	\textbf{Input:} $\mathcal{S},$ $\mathcal{A}$, Time Horizon $T \in \mathcal{Z}$, Number of gradient estimation updates for lower level $K \in \mathcal{Z}$, sample batch size $n \in \mathcal{Z}$, gradient batch size $B \in \mathcal{Z}$, Horizon length $H \in \mathcal{Z}$,  starting policy parameters $\lambda^{0}_{0},{\lambda'}^{0}_{0}$, starting reward parameter $\phi_{0}$ 
		\FOR{$t\in\{0,\cdots,T-1\}$} 
		{   
            \FOR{$k \in\{0,\cdots,K-1\}$} 
		  {
             \STATE $d_{k} = \nabla_{\lambda}\hat{J}({\lambda}_{t}^{k},{\phi}_{t},n,B)$ \\
             \STATE $d^{'}_{k} = \nabla_{\lambda}\hat{J}({{\lambda}^{'}}_{t}^{k},{\phi}_{t},n,B) - {\sigma}.{\nabla}_{\lambda}\hat{G}(\phi_{t},{\lambda^{'}}_{t}^{k},B)$ \\
             \STATE ${\lambda}^{k+1}_{t} = {\lambda}^{k}_{t} + {\tau}{\cdot}\frac{d_{k}}{||d_{k}||}$ \\
             \STATE ${\lambda^{'}}^{k+1}_{t} = {\lambda^{'}}^{k}_{t} + {\tau}^{'}{\cdot}\frac{d^{'}_{k}}{||d^{'}_{k}||}$\\
            }
             \ENDFOR\\
             \STATE $d_{t} = {\nabla}_{\phi}\hat{G}(\phi_{t},{\lambda^{'}}_{t}^{K},B) - \frac{1}{\sigma}\left({\nabla}_{\phi}\hat{J}(\phi_{t},\lambda^{K}_{t},B)-{\nabla}_{\phi}\hat{J}(\phi_{t},{\lambda^{'}}^{K}_{t},B)\right)$ \\
             \STATE $\phi_{t+1} = \phi_{t} - {\eta}{\cdot}{d_{t}}$ \\
		}
		\ENDFOR\\
	\end{algorithmic}
\end{algorithm}


%
\section{Theoretical Analysis}\label{Main Result}

%
We begin by outlining the assumptions required for our analysis, followed by the presentation of our convergence results. We then provide a detailed theoretical analysis, explaining the derivation of these results.
\begin{assump}  \label{assump_1} 
For any $\phi \in \Theta$, $\lambda \in  \Lambda$ and $\sigma \in \mathbb{R}^{+}$, we have the following assumptions
\begin{enumerate}
\item For all $0 \le \sigma \le \sigma_{0}$, the function $h_{\sigma}(\phi,\lambda)$ satisfies the inequality 
     \begin{align}
         ||{\nabla}h_{\sigma}(\phi,\lambda)||^{2} \le {\mu}(h_{\sigma}(\phi,\lambda)-h_{\sigma}(\phi,\lambda_{\sigma}^{*})) \label{as_1}
     \end{align}
     where $\lambda_{\sigma}^{*} = \arg\min_{\lambda \in \Lambda}(h_{\sigma}(\phi,\lambda))$ and $\sigma_{0}$ is a positive constant.

\item The functions $h_{\sigma}(\phi,\lambda)$ and $J(\phi,\lambda)$ are Lipschitz and smooth in variables $\phi$ and $\lambda$.  \label{as_2}

\item The functions $h_{\sigma}(\phi,\lambda)$ and $J(\phi,\lambda)$ have Lipschitz and smooth Hessians in both $\phi$ and $\lambda$.  \label{as_3}
\end{enumerate}
\end{assump}
In \citet{kwon2024penaltymethodsnonconvexbilevel}, the first  Assumption in Equation \eqref{as_1}  was shown to ensure that the proxy objective $\phi_{\lambda}(\phi)$ is differentiable. This assumption also exists in the literature \cite{chen2024findingsmallhypergradientsbilevel} to ensure the existence of the gradient given in Equation \eqref{mdp_ps_10}. It is thus key for the setup given in Equation \eqref{mdp_ps_9} to be solvable using gradient descent.  The Assumption \ref{assump_1}.\ref{as_2} is a standard assumption in bi-level literature used for convergence analyses \citep{grazzi2023bilevel,chen2024findingsmallhypergradientsbilevel}. The Assumption \ref{assump_1}.\ref{as_3} ensures that solving for the optimal point of the proxy objective $\Phi_{\sigma}$ brings us close the optimal point of the true objective $\Phi$. 

\begin{assump}  \label{assump_5} 
For any fixed $ \lambda \in \Lambda, \phi \in \Phi$ and $\theta \in \Theta$ be the parameters of the neural network class used to parametrize the $Q$, where $\Theta$ is a compact set, and $\mu$ is a distribution over $\mathcal{S}\times\mathcal{A}$. Then it holds that
\begin{eqnarray} 
  \min_{\theta \in \Theta}\mathbb{E}_{s,a \sim \mu}\left(Q_{\theta}(s,a) - Q^{\pi_{\lambda}}_{\phi}(s,a) \right)^{2} \le \epsilon_{approx}. \label{assump_5_1} \nonumber
\end{eqnarray}
\end{assump}
 Assumption  \ref{assump_5} ensures that a class of neural networks is able to approximate the function obtained by applying the Bellman operator to a neural network of the same class. Similar assumptions are also considered in \citet{fu2020single,wang2019neural,gaur2024closinggapachievingglobal}. {This assumption ensures that we are able to find an accurate estimate of the $Q$ function. This assumption accounts for the bias in gradient estimation, something not present in the standard bi-level setup. In works such as \cite{shen2023penaltybasedbilevelgradientdescent} a similar constant denoted by $\epsilon_{oracle}$ is used} 

\begin{assump} [For upper level] \label{assump_8} 
For any fixed $\lambda,\lambda_{1},\lambda_{2} \in \Lambda$, $\phi,\phi_{1},\phi_{2} \in \Theta$ and $(s,a) \in \mathcal{S}\times\mathcal{A}$,  we have the following properties
\begin{enumerate} 
\item $||{\nabla}r_{\phi}(s,a)|| \le C_{1} $
\item $||{\nabla}{\log}{\pi}_{\lambda}(s,a)|| \le C_{2}$
\item $||{\nabla}r_{\phi_{1}}(s,a) - {\nabla}r_{\phi_{2}}(s,a)|| \le  C_{3}||\phi_{1} -\phi_{2}||$
\item $||{\nabla}{\log}{\pi}_{\lambda_{1}}(s,a)-{\nabla}{\log}{\pi}_{\lambda_{2}}(s,a)|| \le  C_{4}||\lambda_{1} - \lambda_{1}||$
\end{enumerate}
where $C_{1}-C_{5}$ and $C_{2} \ge 1$ are positive constants. Additionally, there exist $\varepsilon,\bar\varepsilon\in(0,1]$ such that
$\pi_\lambda(a\mid s)\ge \varepsilon$ for all $a\in\mathcal A$ and $\lambda\in\Lambda$, and $\pi_{ref}(a\mid s)\ge \bar\varepsilon$ for all $a\in\mathcal A$ 
\end{assump}
Similar assumptions have been utilized in prior policy gradient-based works  \citep{masiha2022stochastic,mondal2024improved}, as well as actor critic algorithms, such as \citet{fu2020single,gaur2024closinggapachievingglobal,ganesh2025order}.

\begin{assump} [For upper level] \label{assump_6} 
For any fixed $\lambda \in \Lambda$ and $\phi \in \Theta$  we have access to unbiased gradients
\begin{eqnarray} 
 \mathbb{E}[{\nabla}\hat{G}(\phi,\lambda)] = {\nabla}G(\phi,\lambda) \label{assump_6_1} 
\end{eqnarray}
and the gradient estimates have bounded variance
\begin{eqnarray} 
 \mathbb{E}\|{\nabla}\hat{G}(\phi,\lambda)-\mathbb{E}[{\nabla}(G)(\phi,\lambda)]\|^{2} \le \sigma_{G}^{2} \label{assump_6_2} 
\end{eqnarray}
\end{assump}
The assumption for an unbiased gradient with bounded variance is present both in bilevel literature \cite{kwon2024penaltymethodsnonconvexbilevel,chen2024findingsmallhypergradientsbilevel} as well as BRL literature \cite{chakraborty2024parl}. Works such as \citet{shen2024principled} simply assume access to exact gradients of the upper loss function.


{\bf Main Result:} With all the assumptions in place, we are now ready to present the main theoretical results of this work.  
First, we will state the convergence result for Algorithm \ref{algo_1}. 
This result  establishes the  sample complexity bounds for BRL which are the first such results of it's kind. Then, we will go into detail about how these results are obtained, by providing a brief overview of the techniques and lemmas used in establishing the  convergence result.

\begin{theorem} \label{thm}
Suppose Assumptions \ref{assump_1}-\ref{assump_6} hold and we have  $ 0 < \eta \le \frac{1}{2L} $, $0 \le \tau  \le \frac{1}{L_{J}}$, $0 \le \tau^{'} \le \frac{1}{L_{h}}$  where  $L,L_{J},L_{\sigma}$ are the smoothness constants of $\Phi_{\sigma}$,$J$ and $h_{\sigma}$ respectively.  Then from Algorithm \ref{algo_1}, we obtain 
\begin{align}
\frac{1}{T}\sum_{t=1}^{T}\|{\nabla}\Phi(\phi_{t})\|^{2}  \le&  \tilde{\mathcal{O}}\left(\frac{1}{T}\right) +   \tilde{\mathcal{O}}\left(\frac{{\exp}^{-k}}{{\sigma}^{2}}\right) + \tilde{\mathcal{O}}\left(\frac{1}{{\sigma}^{2}{n}}\right) +  \tilde{\mathcal{O}}\left(\frac{{\gamma^{2H}}}{{{\sigma^{2}}B}}\right) + \tilde{\mathcal{O}}(\sigma^{2}) \\ 
&+ \tilde{\mathcal{O}}(\epsilon_{approx}) \label{mdp_main}
\end{align}
If we set $\sigma^{2} =  \tilde{\mathcal{O}}(\epsilon)$, $B =  \tilde{\mathcal{O}}(\epsilon^{-2})$, $n = \tilde{\mathcal{O}}(\epsilon^{-2})$, $T =  \tilde{\mathcal{O}}(\epsilon^{-1})$, $K =  \tilde{\mathcal{O}}(\log\left(\frac{1}{\epsilon}\right))$ and $H =  \tilde{\mathcal{O}}(\log\left(\frac{1}{\epsilon}\right))$ then we obtain 
\begin{align}
\frac{1}{T}\sum_{t=1}^{T}\|{\nabla}\Phi(\phi_{t})\|^{2}  \le& \tilde{\mathcal{O}}(\epsilon) + \tilde{\mathcal{O}}(\epsilon_{approx}) \label{mdp_main_1}
\end{align}
This gives us a sample complexity of $n.K.T +  B.K.H.T +  B.H.T =  \tilde{\mathcal{O}}(\epsilon^{-3})$.
\end{theorem}
Thus we have obtained the first ever sample complexity result for BRL setup. Notably, this result improves on works such as \citet{chakraborty2024parl,shen2024principled} in that our result does not require the state or action space to be finite, while also providing sample complexity and not just iteration complexity results.

\subsection{Proof sketch of Theorem \ref{thm}:} 
The proof is divided into two main parts. The first part is where we establish the local convergence bound of the upper loss function in terms of the error in estimating the gradient of $\Phi_{\sigma}$ as given in Equation \eqref{mdp_ps_11}. This is done using the smoothness assumption on $\Phi$. The next step is to upper bound the error incurred in estimating the gradient of $\Phi_{\sigma}$. The gradient estimation error is shown to be composed of estimating the three terms on the right-hand side of Equation \eqref{mdp_ps_9}. The error in estimating each term is shown to be composed in estimating $\lambda_{\sigma}^{*}(\phi)$ (or $\lambda^{*}(\phi)$) and the error due to having access to an empirical estimate of the gradient. In the estimation of $\lambda_{\sigma}^{*}(\phi)$ (or $\lambda^{*}(\phi)$).  A key insight here is to recognize that in the inner loop of Algorithm \ref{algo_1} we are performing a gradient descent with respect to the parameter $\lambda$ on the functions $J(\phi,\lambda)$ and $h_{\sigma}(\phi,\lambda)$. We use this insight in combination with the PL property from  Assumption \ref{assump_1} to upper bound the error in estimating $\lambda_{\sigma}^{*}(\phi)$ (or $\lambda^{*}(\phi)$).

{\bf Establishing local convergence bound for $\Phi$:} Under Assumption \ref{assump_1}, from the smoothness of $\Phi$, we have
\vspace*{-0.1pt}
\begin{align}
    \Phi(\phi_{t+1}) \le & \Phi(\phi_{t}) + \langle {\nabla}_{\phi}\Phi(\phi_{t}),\phi_{t+1} - \phi_{t} \rangle  + {L}\|\phi_{t+1} - \phi_{t}\|^{2}, \label{proof_1}
\end{align}
 Now, with a step size $\eta \le \frac{1}{2L}$, where $\alpha_{1}$ is the smoothness parameter of $\Phi$,  we get 
\begin{align}
  \frac{1}{T}\sum_{i=1}^{T}\|{\nabla}\Phi(\phi_{t})\|^{2}  \le&   \frac{1}{T}\sum^{t=T}_{t=0}\mathbb{E}\|{\nabla_{\phi}}{\Phi}_{\sigma}(\phi_{t}) - {\nabla_{\phi}}\hat{\Phi}_{\sigma}( \phi_{t})\|^{2} +   \tilde{\mathcal{O}}\left(\frac{1}{T}\right) \nonumber\\
  &+ \tilde{\mathcal{O}}(\sigma^{2}). \label{proof_3_0}
\end{align}
Note that ${\nabla}\hat{\Phi}_{\sigma}$ denotes the empirical estimate of the gradient of the proxy loss function $\Phi_{\sigma}$. Note that we get the term $\tilde{\mathcal{O}}(\sigma^{2})$ using Lemma 4.3 from \citet{chen2024findingsmallhypergradientsbilevel}.

{\bf Gradient estimation error:} The error in the estimation of the gradient at each iteration $k$ of Algorithm \ref{algo_1} given by $\|{\nabla_{\phi}}{\Phi}(\phi_{t}) - {\nabla_{\phi}}\hat{\Phi}_{\sigma}( \phi_{t}))\|$, which is the error between the gradient of the upper objective ${\nabla_{\phi}}{\Phi}(\phi_{t})$ and our estimate of the gradient of the pseudo-objective ${\nabla_{\phi}}\hat{\Phi}_{\sigma}( \phi_{t}))$. This error is decomposed as follows.
\begin{align}
    \underbrace{\mathbb{E}\|{\nabla_{\phi}}{\Phi}_{\sigma}(\phi_{t}) - {\nabla_{\phi}}\hat{\Phi}_{\sigma}(\phi_{t}))\|}_{A_{k}^{'}}&\le \mathbb{E}\| {\nabla_{\phi}}G(\phi_{t},\lambda_{\sigma}^{*}(\phi)) - {\nabla_{\phi}}{G}(\phi_{t},{\lambda^{'}}^{K}_{t},B)\|  \nonumber\\ 
    &\hspace{-2.2cm}+ \frac{1}{\sigma}\mathbb{E}\|{\nabla_{\phi}}J({\phi}_{t},{\lambda^{*}(\phi)}) - {\nabla_{\phi}}{J}({\phi}_{t},{\lambda^{K}_{t}},B)\| \nonumber \\
    &\hspace{-2cm}+ \frac{1}{\sigma}\mathbb{E}\| {\nabla_{\phi}}J(\phi_{t},{\lambda_{\sigma}^{*}}(\phi))-  {\nabla_{\phi}}{J}({{\phi}_{t},{\lambda^{'}}^{K}_{t}},B)\|.  \label{proof_3_1}
\end{align}
Thus, the error incurred in the estimation of the gradient terms can be broken into the error in estimation of the three terms, ${\nabla}G(\phi_{t},\lambda_{\sigma}^{*}(\phi_{t}))$, ${\nabla}J({\phi_{t}},{\lambda^{*}(\phi_{t})})$ and $ {\nabla}J(\phi_{t},\lambda_{\sigma}^{*}(\phi_{t}))$. We first focus on the estimation error for the term ${\nabla_{\phi}}J(\phi,{\lambda_{\sigma}^{*}}(\phi))$ where the error in estimation can be decomposed as
\begin{align}
    \mathbb{E}\| {\nabla_{\phi_{t}}}J(\phi_{t},\lambda_{\sigma}^{*}(\phi_{t})) - {\nabla_{\phi}}{J}(\phi_{t},{\lambda^{'}}^{K}_{t},B)\| &\le  \mathbb{E}\|{\nabla_{\phi}}J(\phi_{t},\lambda_{\sigma}^{*}(\phi_{t})) - {\nabla_{\phi}}J(\phi_{t},{\lambda^{'}}^{K}_{t})\| \nonumber\\
    &+  \mathbb{E}\|{\nabla_{\phi}}J(\phi_{t},{\lambda^{'}}^{K}_{t}) -   {\nabla_{\phi}}{J}(\phi_{t},{\lambda^{'}}^{K}_{t},B)\|.\label{proof_3_2}
\end{align}
The second term on the right-hand side of Equation \eqref{proof_3_2} is the error incurred due to the difference between the gradient of $J$ and its empirical estimate. This error is upper bounded using the defintion of the gradient given in Equation \eqref{mdp_ps_12}. 

The first term on the right-hand side is the error incurred due to the error in estimating $\lambda_{\sigma}^{*}(\phi)$. In order to show this, we write the following
\begin{align}
    \mathbb{E}\|{\nabla_{\phi}}J(\phi_{t},\lambda_{\sigma}^{*}(\phi_{t})) - {\nabla_{\phi}}J(\phi_{t},{\lambda'}^{K}_{t})\|^{2} &\le L_{J}\mathbb{E}\|\lambda_{\sigma}^{*}(\phi_{t})-{\lambda^{'}}^{K}_{t}\|^{2} \label{proof_3_2_1}\\ 
    &\le L_{\sigma}{\cdot}{\mu}\mathbb{E}|h_{\sigma}(\phi_{t},\lambda_{\sigma}^{*}(\phi_{t}))-h_{\sigma}(\phi_{t},{\lambda^{'}}^{K}_{t})| \label{proof_3_2_2}.
\end{align}
We get Equation \eqref{proof_3_2_1} from the smoothness of $J(\phi,\lambda)$ assumed in Assumption \ref{assump_1}. We get Equation \eqref{proof_3_2_2} from  Equation \eqref{proof_3_2_1} by using the quadratic growth property of PL functions applied to $h_{\sigma}(\phi,\lambda)$ also assumed in Assumption \ref{assump_1}.

{In order to bound the right hand side of Equation \eqref{proof_3_2_2}, we establish the following result.}

\begin{lemma} \label{thm_supp}
Consider an $ L$-smooth differentiable function denoted by $f(\lambda)$ satisfying the PL property with PL constant $\mu$. If we apply the stochastic gradient descent with step size $0 \le \eta \le \frac{1}{L}$, then we obtain the following
\begin{align}
    (f(\lambda_{k}) -f(\lambda^{*})) &= \tilde{\mathcal{O}}\left(e^{-k} \right) + \mathcal{O}(\beta(n,B,H))\label{proof_3_3}
\end{align}
where $\forall \lambda \in \Lambda$, $\beta(n)$ satisfies 
\begin{align}
\mathbb{E}\|\nabla_{\lambda}f(\lambda_{k}) - \nabla_{\lambda}\hat{f}(\lambda_{k})  \|^{2} &\le \beta(n,B,H)  
\end{align}
 and ${\nabla}_{\lambda}\hat{f}(\lambda)$ denotes the estimate of ${\nabla}_{\lambda}f(\lambda)$ and $\lambda^{*} = argmin_{\lambda \in \Lambda}f(\lambda)$.
\end{lemma} 
This result is obtained using a recursive analysis of the optimality gap when performing an SGD in the presence of biased gradient estimates. Using this lemma, we can bound the right-hand side of Equation \eqref{proof_3_2_2} in terms of error in estimating the gradient of $h_{\sigma}$ with respect to $\lambda$. Thus, we obtain
\begin{align}
    \mathbb{E}|h_{\sigma}(\phi_{t},\lambda_{\sigma}^{*}(\phi_{t}))-h_{\sigma}(\phi_{t},{\lambda^{'}}^{K}_{t})|  &\le \tilde{\mathcal{O}}\left(e^{-K} \right) + \mathcal{O}(\beta(n,B,H)) \nonumber\\
     \label{proof_3_4}
\end{align}
where  $\forall \lambda \in \Lambda$, $\beta(n,B,H)$ satisfies  $\mathbb{E}\|{\nabla}_{\lambda}h_{\sigma}(\phi_{t},{\lambda})-{\nabla}_{\lambda}\hat{h}_{\sigma}(\phi_{t},{\lambda})\|^{2} \le \beta(n,B,H)$.
Using the expression for gradients of $J(\phi,\lambda)$ and $G(\phi,\lambda)$  we are able obtain the following result
\begin{align}
    \mathbb{E}\|{\nabla_{\phi}}J(\phi_{t},\lambda_{\sigma}^{*}(\phi_{t})) - {\nabla_{\phi}}J(\phi_{t},{\lambda'}^{K}_{t})\|^{2} 
    &\le \tilde{\mathcal{O}}\left(e^{-k}\right)  + \tilde{\mathcal{O}}\left(\frac{\gamma^{2H}}{{B}}\right) + \tilde{\mathcal{O}}\left(\frac{1}{{n}}\right) \nonumber\\
    &+ \tilde{\mathcal{O}}(\epsilon_{approx})  \label{proof_3_5}
\end{align}
where $n$ is the number of samples used to estimate the $Q$ function. The details of this are given in Lemma \ref{lemma_4} of the Appendix. 
For upper bounding the other two terms on the right-hand side of Equation \eqref{proof_3_1}, we use a similar decomposition and analysis. These are described in detail in Lemma \ref{lemma_2} and  Lemma \ref{lemma_3} of the Appendix. 
Finally, plugging the obtained expressions back into the right-hand side of Equation \eqref{proof_3_1} and the resulting expression into the right-hand side of Equation \eqref{proof_3_0} gives us Theorem \ref{thm}. We provide an evaluation of Algorithm \ref{algo_1} in Appendix \ref{appendix:experiments}.



\section{Standard Bilevel Optimization: A Special Case} \label{bi_level_standard}

In this section, we show how the techniques used to establish Theorem \ref{thm} can also yield a state-of-the-art sample complexity result for standard bilevel optimization with a non-convex lower level (where the lower level is not an RL problem). The key distinction between our BRL setup and standard bilevel optimization is that it is assumed that we have access to unbiased gradients with bounded variance \citep{kwon2024penaltymethodsnonconvexbilevel, chen2024findingsmallhypergradientsbilevel}. \C{\textcolor{red}{BE MORE SPECIFIC}} This is not the case in the BRL setup as discussed in Section \ref{Proposed Algorithm}.  We show that assuming access to unbiased gradients with bounded variance enables achieving a state-of-the-art sample complexity result for bilevel optimization.


The bilevel optimization problem is similar to \eqref{mdp_ps_6}, and is given as
\begin{align}
    \min_{\phi}{\Phi}(\phi)&:= G(\phi,\lambda \in \Lambda^{*}(\phi)), \nonumber\\
    &\text{ where }\Lambda^{*} \in \arg\min_{\lambda} -J({\phi},{\lambda}). \label{bl_1}
\end{align}


As before, we solve the proxy problem in Equation \eqref{mdp_ps_9} using gradient descent with the gradient expression from Equation \eqref{mdp_ps_10}. The key difference here is the availability of unbiased gradients for both the upper- and lower-level loss functions, as captured in the following assumption.

\begin{assump}  \label{assump_7} 
For any fixed $\lambda \in \Lambda$ and $\phi \in \Theta$  we have access to unbiased gradients
\begin{eqnarray} 
 \mathbb{E}[{\nabla}\hat{G}(\phi,\lambda)] = {\nabla}G(\phi,\lambda), \\
  \mathbb{E}[{\nabla}\hat{J}(\phi,\lambda,)]= {\nabla}G(\phi,\lambda)\label{assump_7_1}
\end{eqnarray}
and the gradient estimates have bounded variance
\begin{eqnarray} 
 \mathbb{E}\|{\nabla}\hat{G}(\phi,\lambda)-\mathbb{E}{\nabla}(G)(\phi,\lambda)\|^{2} \le \sigma_{G}^{2}, \\ 
  \mathbb{E}\|{\nabla}\hat{J}(\phi,\lambda)-\mathbb{E}{\nabla}(G)(\phi,\lambda)\|^{2} \le \sigma_{J}^{2} \label{assump_7_2}
\end{eqnarray}
\end{assump}

This provides the gradient estimate for the lower-level loss function, and  Equation \eqref{mdp_ps_13} is the gradient estimate for the upper-level loss function. Here, \( {\nabla}\hat{J}_{i}(\phi,\lambda) \) are independent sampled unbiased estimates of \( {\nabla}{J}(\phi,\lambda) \), and \( B \) represents the batch size. We assume that these samples of the estimate can be independently sampled. Additionally, we assume that this can be done for the gradient with respect to both $\lambda$ and $\phi$. This is in line with other BRL works such as \citet{chakraborty2024parl,shen2024principled}. We also define the following term
\begin{align}
     {\nabla}{J}(\phi,\lambda,B) = \frac{1}{B} \sum_{j=1}^{B}{\nabla}\hat{J}_{j}(\phi,\lambda). \label{bl_2}
\end{align}
which is what we use instead of ${\nabla}{J}(\phi,\lambda,n,B)$ in Algorithm \ref{algo_1} for the standard bi-level setup. 
For a bi-level optimization with a non-convex lower level, we obtain
\begin{theorem} \label{thm_bi}
Suppose Assumptions \ref{assump_1} and \ref{assump_7} hold and we have  $ 0 < \eta \le \frac{1}{2L} $, $0 \le \tau_{k}  \le \frac{1}{L_{J}}$, $0 \le \tau_{k}  \le \frac{1}{L_{h}}$  where  $L,L_{J},L_{\sigma}$ are the smoothness constants of $\Phi_{\sigma}$,$J$ and $h_{\sigma}$ respectively. We further replace $\nabla_{\lambda}{J}({\phi},{\lambda},n,B)$ with ${\nabla}{J}(\phi,\lambda,B)$ as defined in \eqref{bl_2}. Then, from Algorithm \ref{algo_1}  we obtain 
\begin{align}
\frac{1}{T}\sum_{t=1}^{T}\|{\nabla}\Phi(\phi_{t})\|^{2}    &\le  \tilde{\mathcal{O}}\left(\frac{1}{T}\right) +   \tilde{\mathcal{O}}\left(\frac{{\exp}^{-k}}{{\sigma}^{2}}\right) + \tilde{\mathcal{O}}\left(\frac{1}{{\sigma}^{2}{B}}\right) + \tilde{\mathcal{O}}(\sigma^{2})
\end{align}
If we set $\sigma^{2} =  \tilde{\mathcal{O}}(\epsilon)$, $B = \tilde{\mathcal{O}}(\epsilon^{-2})$, $T =  \tilde{\mathcal{O}}(\epsilon^{-1})$, $K =  \tilde{\mathcal{O}}(\log\left(\frac{1}{\epsilon}\right))$.
\begin{align}
\frac{1}{T}\sum_{t=1}^{T}\|{\nabla}\Phi(\phi_{t})\|^{2}  \le \tilde{\mathcal{O}}(\epsilon)  \label{mdp_main_2}
\end{align}
This gives us a sample complexity of $B.K.T + B.T  = \tilde{\mathcal{O}}(\epsilon^{-3})$.
\end{theorem}

%
%

\begin{table}[t]
\centering
 \caption{ If we assume access to unbiased gradients, we obtain a state of the art sample complexity of $\epsilon^{-3}$  for bilevel optimization without lower level convexity restriction.\vspace{2mm}}
\label{tbl_related_2}%
{
{\begin{tabular}{|c|c|c|c|c|c|}
\hline
References & \multicolumn{1}{c|}{\begin{tabular}[c]{@{}c@{}}  Non-convex LL \end{tabular}} & \multicolumn{1}{c|}{\begin{tabular}[c]{@{}c@{}} Without\\ second order \end{tabular}} &   \multicolumn{1}{c|}{\begin{tabular}[c]{@{}c@{}} Iteration \\complexity \end{tabular}} &  \multicolumn{1}{c|}{\begin{tabular}[c]{@{}c@{}} Sample \\ complexity \end{tabular}} \\ \hline
     \citet{ji2021bilevel}     &                    \textcolor{red}{\xmark}               &               \textcolor{red}{\xmark}    & $\tilde{\mathcal{O}}(\epsilon^{-1})$        &        $\tilde{\mathcal{O}}(\epsilon^{-2})$                                                       \\
     \citet{sow2022convergence}     &           \textcolor{red}{\xmark}                        &               \textcolor{green}{\cmark}          &    $\tilde{\mathcal{O}}(\epsilon^{-2})$        &        $\tilde{\mathcal{O}}(\epsilon^{-4})$                                                       \\ 
    \citet{stoch_bil2}     &                    \textcolor{red}{\xmark}               &              \textcolor{green}{\cmark}          &  
    $\tilde{\mathcal{O}}(\epsilon^{-\frac{5}{2}})$ &        $\tilde{\mathcal{O}}(\epsilon^{-\frac{5}{2}})$                                                       \\ 
    \citet{yang2023achieving}  &                    \textcolor{red}{\xmark}                 &             \textcolor{green}{\cmark}         &     $\tilde{\mathcal{O}}(\epsilon^{-\frac{3}{2}})$ &        $\tilde{\mathcal{O}}(\epsilon^{-\frac{3}{2}})$                                                       \\ 
        \citet{kwon2024penaltymethodsnonconvexbilevel}  & \textcolor{green}{\cmark}                 &       \textcolor{green}{\cmark}          & $\tilde{\mathcal{O}}(\epsilon^{-5})$ &        $\tilde{\mathcal{O}}(\epsilon^{-7})$                                                       \\   \citet{chen2024findingsmallhypergradientsbilevel}     &                    \textcolor{green}{\cmark}                &               \textcolor{green}{\cmark}  & $\tilde{\mathcal{O}}(\epsilon^{-2})$ &        $\tilde{\mathcal{O}}(\epsilon^{-6})$                                                       \\          
        This Work &                    \textcolor{green}{\cmark}                &               \textcolor{green}{\cmark}          &  $\tilde{\mathcal{O}}(\epsilon^{-1})$ &        $\tilde{\mathcal{O}}(\epsilon^{-3})$             \\    
    \hline
\end{tabular}
}}
\end{table}

Note the absence of the term $\mathcal{O}({\epsilon_\text{approx}})$ as we have assumed access to unbiased gradient estimates for both upper and lower loss functions. As noted earlier, our result advances previous analyses of bi-level optimization with non-convex lower levels. \citet{kwon2024penaltymethodsnonconvexbilevel} established a sample complexity of $\mathcal{O}(\epsilon^{-7})$, later improved to $\mathcal{O}(\epsilon^{-6})$ by \citet{chen2024findingsmallhypergradientsbilevel}. Table \ref{tbl_related_2} highlights how our approach enhances existing results in bi-level optimization and brings convergence results from non-convex lower level setups to those of convex lower level setups such as \cite{grazzi2023bilevel,yang2023achieving}.

\section{Conclusion} \label{Conclusion}

This paper established the first sample complexity bounds for bilevel reinforcement learning (BRL) in parameterized settings, achieving $ O(\epsilon^{-3})$. Our approach, leveraging penalty-based formulations and first-order methods, improves scalability without requiring costly Hessian computations. These results extend to standard bilevel optimization, setting a new state-of-the-art for non-convex lower-level problems. Our work provides a foundation for more efficient BRL algorithms with applications in AI alignment and RLHF. Future direction include improving the theoretical bounds in this paper, and evaluating the proposed algorithm in different applications.

\section{Acknowledgment}
The work was supported in part by the National Science Foundation under  grant CCF-2149588 and Cisco Systems, Inc. 

\bibliography{iclr2025_conference}
\bibliographystyle{plain}
\newpage
\section*{NeurIPS Paper Checklist}

\begin{enumerate}

\item {\bf Claims}
    \item[] Question: Do the main claims made in the abstract and introduction accurately reflect the paper's contributions and scope?
    \item[] Answer: 
    \answerYes{}
    \item[] Justification: The claims are demonstrated in the key results in Lemmas and Theorems, with explanations next to them. 

\item {\bf Limitations}
    \item[] Question: Does the paper discuss the limitations of the work performed by the authors?
    \item[] Answer: \answerYes{} 
    \item[] Justification:  The assumptions given in the paper give the limitations of this work. Further, future work direction in the conclusions describe another limitation of this work. 

\item {\bf Theory Assumptions and Proofs}
    \item[] Question: For each theoretical result, does the paper provide the full set of assumptions and a complete (and correct) proof?
    \item[] Answer: \answerYes{} 
    \item[] Justification: We have provided the assumptions used in the work at one place, which are used in all the results.  

    \item {\bf Experimental Result Reproducibility}
    \item[] Question: Does the paper fully disclose all the information needed to reproduce the main experimental results of the paper to the extent that it affects the main claims and/or conclusions of the paper (regardless of whether the code and data are provided or not)?
    \item[] Answer: \answerYes{} 
    \item[] Justification: Details provided in Appendix \ref{appendix:experiments}. 

\item {\bf Open access to data and code}
    \item[] Question: Does the paper provide open access to the data and code, with sufficient instructions to faithfully reproduce the main experimental results, as described in supplemental material?
    \item[] Answer: \answerYes{} 
    \item[] Justification: Details provided in Appendix \ref{appendix:experiments}.

\item {\bf Experimental Setting/Details}
    \item[] Question: Does the paper specify all the training and test details (e.g., data splits, hyperparameters, how they were chosen, type of optimizer, etc.) necessary to understand the results?
    \item[] Answer: \answerYes{} 
    \item[] Justification: Details provided in Appendix \ref{appendix:experiments}.

\item {\bf Experiment Statistical Significance}
    \item[] Question: Does the paper report error bars suitably and correctly defined or other appropriate information about the statistical significance of the experiments?
    \item[] Answer: \answerYes{} 
    \item[] Justification: Details provided in Appendix \ref{appendix:experiments}.

\item {\bf Experiments Compute Resources}
    \item[] Question: For each experiment, does the paper provide sufficient information on the computer resources (type of compute workers, memory, time of execution) needed to reproduce the experiments?
    \item[] Answer: \answerYes{} 
    \item[] Justification: Details provided in Appendix \ref{appendix:experiments}.
    
\item {\bf Code Of Ethics}
    \item[] Question: Does the research conducted in the paper conform, in every respect, with the NeurIPS Code of Ethics \url{https://neurips.cc/public/EthicsGuidelines}?
    \item[] Answer: \answerYes{} 
    \item[] Justification: All the points mentioned in the NeurIPS Code of Ethics are taken into consideration.

\item {\bf Broader Impacts}
    \item[] Question: Does the paper discuss both potential positive societal impacts and negative societal impacts of the work performed?
    \item[] Answer: \answerNA{} 
    \item[] Justification: Since the work is primarily theoretical in nature, no potential negative societal impact.
    
\item {\bf Safeguards}
    \item[] Question: Does the paper describe safeguards that have been put in place for responsible release of data or models that have a high risk for misuse (e.g., pretrained language models, image generators, or scraped datasets)?
    \item[] Answer: \answerNA{} 
    \item[] Justification: The paper poses no such risks. 

\item {\bf Licenses for existing assets}
    \item[] Question: Are the creators or original owners of assets (e.g., code, data, models), used in the paper, properly credited and are the license and terms of use explicitly mentioned and properly respected?
    \item[] Answer: \answerYes{} 
    \item[] Justification: All existing works used are properly credited. 

\item {\bf New Assets}
    \item[] Question: Are new assets introduced in the paper well documented and is the documentation provided alongside the assets?
    \item[] Answer: \answerYes{} 
    \item[] Justification:  Details provided in Appendix \ref{appendix:experiments}

\item {\bf Crowdsourcing and Research with Human Subjects}
    \item[] Question: For crowdsourcing experiments and research with human subjects, does the paper include the full text of instructions given to participants and screenshots, if applicable, as well as details about compensation (if any)? 
    \item[] Answer: \answerNA{} 
    \item[] Justification: The paper does not involve crowdsourcing nor research with human subjects.

\item {\bf Institutional Review Board (IRB) Approvals or Equivalent for Research with Human Subjects}
    \item[] Question: Does the paper describe potential risks incurred by study participants, whether such risks were disclosed to the subjects, and whether Institutional Review Board (IRB) approvals (or an equivalent approval/review based on the requirements of your country or institution) were obtained?
    \item[] Answer: \answerNA{} 
    \item[] Justification: The paper does neither involve crowd-sourcing nor research with human subjects.
\end{enumerate}
\onecolumn
\appendix
\section{Proof of Lemma \ref{thm_supp}} \label{lem_proof}
\begin{proof} 

\[
\lambda_{t+1} \;=\; \lambda_t \;-\; \eta\, \nabla\hat f(\lambda_t),
\]
where the biased stochastic gradient 
\[
b_t \;\coloneqq\; \mathbb{E}\!\left[\nabla\hat f(\lambda_t)\right]-\nabla f(\lambda_t),
\qquad 
\mathbb{E}\!\Big[\big\|\nabla\hat f(\lambda_t)-\mathbb{E}[\nabla\hat f(\lambda_t)]\big\|^2\Big] = \sigma_{t}^2.
\]
By $L$-smoothness, for $y=\lambda_{t+1}=\lambda_t-\eta\,\nabla\hat f(\lambda_t)$,
\begin{align*}
f(\lambda_{t+1})
&\le f(\lambda_t) + \big\langle \nabla f(\lambda_t),\, \lambda_{t+1}-\lambda_t\big\rangle
+ \frac{L}{2}\,\|\lambda_{t+1}-\lambda_t\|^2 \\
&= f(\lambda_t) - \eta\,\big\langle \nabla f(\lambda_t),\, \nabla\hat f(\lambda_t)\big\rangle
+ \frac{L}{2}\,\eta^2 \,\big\|\nabla\hat f(\lambda_t)\big\|^2 .
\end{align*}
Taking expectation and using the bias notation,
\begin{align*}
\mathbb{E}\!\left[f(\lambda_{t+1})\right]
&\le f(\lambda_t) 
- \eta\,\big\langle \nabla f(\lambda_t),\, \nabla f(\lambda_t)+ b_t \big\rangle
+ \frac{L}{2}\,\eta^2 \,\mathbb{E}\!\left[\big\|\nabla\hat f(\lambda_t)\big\|^2 \right] \\
&= f(\lambda_t) 
- \eta\,\|\nabla f(\lambda_t)\|^2
- \eta\,\big\langle \nabla f(\lambda_t),\, b_t \big\rangle
+ \frac{L}{2}\,\eta^2 \,\mathbb{E}\!\left[\big\|\nabla\hat f(\lambda_t)\big\|^2 \right].
\end{align*}
Apply Cauchy--Schwarz and Young's inequality to the bias inner product:
\[
-\eta\,\langle \nabla f(\lambda_t), b_t\rangle
\;\le\; \eta\,\|\nabla f(\lambda_t)\|\,\|b_t\|
\;\le\; \frac{\eta}{2}\,\|\nabla f(\lambda_t)\|^2 \;+\; \frac{\eta}{2}\,\|b_t\|^2
\]
For the second moment, decompose around the (biased) mean:

Combining the bounds yields
\begin{align}
&\mathbb{E}\!\left[f(\lambda_{t+1})\right] \\
&\le f(\lambda_t)
+ \Big(-\tfrac{\eta}{2}+L\eta^2\Big)\,\|\nabla f(\lambda_t)\|^2
+ \frac{\eta}{2}\,||b_{t}||^2
+ \frac{L}{2}\,\eta^2\,(2b^{2}_{t}+\sigma^2).\\
&\le f(\lambda_t)
+ \Big(-\tfrac{\eta}{2}+L\eta^2\Big)\,\|\nabla f(\lambda_t)\|^2
+ \frac{\eta}{2}\,\big(||b_{t}||^2 + {\sigma}^{2}  \big)
+ \frac{L}{2}\,\eta^2\,(2b^{2}_{t}+2\sigma_{t}^2). \label{lem1_1_1}\\
&\le f(\lambda_t)
+ \Big(-\tfrac{\eta}{2}+L\eta^2\Big)\,\|\nabla f(\lambda_t)\|^2
+ \frac{\eta}{2}\big(\mathbb{E}||{\nabla}\hat{f}(\lambda_{t}) - {\nabla}{f}(\lambda_{t})||^{2}  \big) \label{lem1_1_2}\\
&+ \frac{L}{2}\,\eta^2\,2(\mathbb{E}||{\nabla}\hat{f}(\lambda_{t}) - {\nabla}{f}(\lambda_{t})||^{2}).\\
&\le f(\lambda_t)
+ \Big(-\tfrac{\eta}{2}+L\eta^2\Big)\,\|\nabla f(\lambda_t)\|^2
+ \frac{\eta}{2}\,
\big( \beta(n,B,H) \big)
+ \frac{L}{2}\,\eta^2\,(2\beta(n,B,H)) \label{lem1_1_3}
\end{align}

We obtain equation \eqref{lem1_1_2} from equation \ref{lem1_1_1} by using the identity that the square of the bias plus variance is equal to the mean square error. We obtain \eqref{lem1_1_3} from equation \ref{lem1_1_2} from the definition of $\beta(n,B,H)$.  
\begin{align}
  f(\lambda_{t+1}) \le f(\lambda_{t}) - \left(\eta - \frac{L.{\eta}^{2}}{2}\right) ||{\nabla}f(\lambda_{t})||^2 + \frac{(L.{\eta}^{2} + \eta) \beta(n,B,H) }{2}
\end{align}

Now applying the PL inequality (Assumption \ref{assump_1}), $\|\nabla f(\lambda_t)\|^2 \ge 2\mu \left(f(\lambda_t) - f^* \right)$, we substitute in the above inequality to get
\begin{align}
    f(\lambda_{t+1})- f^* 
    &\le \left(1 - 2\mu \left(\eta - \frac{L\eta^2}{2} \right)\right)\left(f(\lambda_t) - f^*\right) + \frac{(L.{\eta}^{2} + \eta) { \beta(n,B,H) }}{2}.
\end{align}

Define the contraction factor
\begin{align}
    \rho \coloneqq 1 - 2\mu \left(\eta - \frac{L.{\eta}^{2}}{2} \right).
\end{align}

 we get the recursion:
\begin{align}
\delta_{t+1} \le \rho \cdot \delta_t + \frac{(L.{\eta}^{2} + \eta).\beta(n,B,H)}{2}.
\end{align}
When $ \eta \le \frac{1}{L} $, we have
\begin{align}
\eta - \frac{L \eta^2}{2} \ge \frac{\eta}{2} \Rightarrow \rho \le 1 - \mu \eta.
\end{align}
Unrolling the recursion we have
\begin{align}
\delta_t \le (1 - \mu \eta)^t \delta_0 + \frac{(L.{\eta^2}+ \eta)\beta(n,B,H)}{2} \sum_{j=0}^{t-1} (1 - \mu \eta)^j.
\end{align}

Using the geometric series bound:
\begin{align}
    \sum_{j=0}^{t-1} (1 - \mu \eta)^j \le \frac{1}{\mu \eta},
\end{align}
we conclude that
\begin{align}
    \delta_t \le (1 - \mu \eta)^t \delta_0 + \frac{(L.{\eta}^{2} + \eta) \beta(n,B,H)}{2\mu{\cdot}{\eta}}.
\end{align}
Hence, we have the convergence result
\begin{align}
    f(\lambda_t) - f^* \le (1 - \mu\eta)^t \delta_0 + \frac{(L.{\eta} + \eta){\cdot}\beta(n,B,H)}{2\mu{\cdot}{\eta}}.
\end{align}
\end{proof}

\begin{lemma}[Uniform bound for a sample-based KL gradient estimator] \label{lem_kl}
Let $\mathcal A$ be a action space and, for a fixed state $s$, let $\pi_\lambda(\cdot\mid s)$ and $\bar\pi(\cdot\mid s)$ be two policies on $\mathcal A$,
with parameter $\lambda \in \Lambda$.
Assume:
\begin{enumerate}
\item[(i)] (\emph{Bounded score}) There exists $B<\infty$ such that
$\|\nabla_{\lambda}  \log \pi_\theta(a\mid s)\|\le B$ for all $a\in\mathcal A$ and $\theta\in\Theta$.
\item[(ii)] (\emph{Common support bounded away from $0$}) There exist $\varepsilon,\bar\varepsilon\in(0,1]$ such that
$\pi_\lambda(a\mid s)\ge \varepsilon$ for all $a\in\mathcal A$ and $\lambda\in\Lambda$, and $\pi_{ref}(a\mid s)\ge \bar\varepsilon$ for all $a\in\mathcal A$.
\end{enumerate}
Define the per-sample contribution
\[
g_\theta(s,a)
\;:=\;
\nabla_{\lambda}  \log \pi_\theta(a\mid s)\,
\Big(1+\log\pi_\theta(a\mid s)-\log\pi_{ref}(a\mid s)\Big),
\]
so that
$\nabla_{\lambda}  D_{\mathrm{KL}}(\pi_\theta\Vert \pi_{ref})
=\mathbb{E}_{a\sim\pi_\theta(\cdot\mid s)}[\,g_\theta(s,a)\,]$.
Then, with $C_{\log}:=\log(1/\varepsilon)+\log(1/\bar\varepsilon)$,
\[
\|g_\theta(s,a)\|\;\le\; B\,(1+C_{\log})
\qquad \text{for all $a\in\mathcal A$ and $\theta\in\Theta$},
\]
and consequently, for any $n\ge 1$ and i.i.d.\ draws $a_1,\dots,a_n\sim\pi_\theta(\cdot\mid s)$, the Monte-Carlo estimator
$\hat g_n:=\tfrac{1}{n}\sum_{i=1}^n g_\theta(s_{i},a_i)$ satisfies
$\|\hat g_n\|\le B\,(1+C_{\log})$.
\end{lemma}

\begin{proof}
(i) and (ii) are satisfied from Assumption \ref{assump_8}, for every $a\in\mathcal A$ and $\theta\in\Theta$,
$\pi_\theta(a\mid s)\in[\varepsilon,1]$ and $\pi_{ref}(a\mid s)\in[\bar\varepsilon,1]$,
hence $\log\pi_\theta(a\mid s)\in[\log\varepsilon,0]$ and $\log\pi_{ref}(a\mid s)\in[\log\bar\varepsilon,0]$.
Therefore
\[
\big|\log\pi_\theta(a\mid s)-\log\pi_{ref}(a\mid s)\big|
\;\le\; \log(1/\varepsilon)+\log(1/\bar\varepsilon)
\;=\; C_{\log},
\]
and thus $\big|1+\log\pi_\theta(a\mid s)-\log\pi_{ref}(a\mid s)\big|\le 1+C_{\log}$.
By (i),
\[
\|g_\theta(s,a)\|
=\big\|\nabla_{\lambda}  \log \pi_\theta(a\mid s)\big\|\,\big|1+\log\pi_\theta(a\mid s)-\log\pi_{ref}(a\mid s)\big|
\le B\,(1+C_{\log}).
\]
This bound is deterministic (independent of the sample index) and holds for all $a,\theta$,
so taking averages over  samples preserves it:
$\|\hat g_n\|\le B\,(1+C_{\log})$.
\end{proof}

\section{Proof of Theorem \ref{thm}} \label{thm_proof}
\begin{proof}

Using same steps as for Lemma \ref{lem_proof} we get

\begin{eqnarray}
   \Phi(\phi_{t+1}) &\le&  \Phi(\phi_{t}) - \left(\frac{\eta}{2} - \frac{L.{\eta}^{2}}{2} \right)||\nabla{\Phi}(\phi_{t})||^{2} + \frac{\eta+L{\eta}^{2}}{2}{\mathbb{E}}||{\nabla_{\phi}}{\Phi}(\phi_{k}) - {\nabla_{\phi}}\hat{\Phi}_{\sigma}( \phi_{k})||^{2} \nonumber\\
   &&\label{main_proof_3}
\end{eqnarray}

Now rearranging terms, summing Equation \eqref{main_proof_3} over $T$ and dividing by $T$ on both sides we get
\begin{eqnarray}
  \frac{1}{T}\sum_{t=1}^{T}||{\nabla}\Phi(\phi_{t})||^{2}  &\le&   \frac{1}{T}\sum^{t=T}_{t=0}\underbrace{{\mathbb{E}}||{\nabla_{\phi}}{\Phi}(\phi_{k}) - {\nabla_{\phi}}\hat{\Phi}_{\sigma}( \phi_{k})||^{2}}_{A_{t}} + \tilde{\mathcal{O}}\left(\frac{1}{T}\right). \label{main_proof_4}
\end{eqnarray}

We now bound $A_{t}$ as follows
\begin{eqnarray}
    {\mathbb{E}}||{\nabla_{\phi}}{\Phi}(\phi_{t}) - {\nabla_{\phi}}\hat{\Phi}_{\sigma}(\phi_{t}))|| &=& {\mathbb{E}}||{\nabla_{\phi}}{\Phi}(\phi_{t}) - {\nabla_{\phi}}{\Phi}_{\sigma}(\phi_{t}) + {\nabla_{\phi}}{\Phi}_{\sigma}(\phi_{t}) - {\nabla_{\phi}}\hat{\Phi}_{\sigma}(\phi_{t}))||, \nonumber\\
    &&  \label{main_proof_5}\\
    &\le& {\mathbb{E}}||{\nabla_{\phi}}{\Phi}(\phi_{t}) - {\nabla_{\phi}}{\Phi}_{\sigma}(\phi_{t}))|| \nonumber\\
    &+& {\mathbb{E}}||{\nabla_{\phi}}{\Phi}_{\sigma}(\phi_{t}) - {\nabla_{\phi}}\hat{\Phi}_{\sigma}(\phi_{t}))||,  \label{main_proof_5_1}\\
    &\le&  \mathcal{O}(\sigma) +  {\mathbb{E}}\underbrace{||{\nabla_{\phi}}{\Phi}_{\sigma}(\phi_{t}) - {\nabla_{\phi}}\hat{\Phi}_{\sigma}(\phi_{t}))||}_{A_{t}}, \label{main_proof_6}
\end{eqnarray}
The first term on the right-hand side denotes the gap between the gradient of the objective function and the gradient of the pseudo-objective $\Phi_{\sigma}$. We get the upper bound on this term from Lemma  4.3 of \cite{chen2024findingsmallhypergradientsbilevel}. The term $A^{'}_{t}$ denotes the error incurred in estimating the true gradient of the pseudo-objective. 
\begin{eqnarray}
    \underbrace{{\mathbb{E}}||{\nabla_{\phi}}{\Phi}_{\sigma}(\phi_{t}) - {\nabla_{\phi}}\hat{\Phi}_{\sigma}(\phi_{t}))||^{2}}_{A_{t}}
    &\le&  {\mathbb{E}}\Bigg|\Bigg| {\nabla_{\phi}}G(\phi_{t},\lambda_{\sigma}^{*}(\phi_{t}))  +  \frac{{\nabla_{\phi}}J({\phi_{t}},\lambda^{*}(\phi_{t})) -  {\nabla_{\phi}}J(\phi_{t},{\lambda_{\sigma}^{*}}(\phi_{t}))}{\sigma} \nonumber\\
    &-& {\nabla_{\phi}}{G}(\phi_{t},\lambda^{K}_{t},B)  +  \frac{{\nabla_{\phi_{t}}}\hat{J}({\phi_{t}},\lambda^{K}_{t}) -  {\nabla_{\phi}}{J}({\phi_{t}},{\lambda^{'}}^{K}_{t}(\phi)),B}{\sigma}\Bigg|\Bigg|^{2}, \nonumber\\
    && \label{main_proof_7}\\
    &\le& {\mathbb{E}}|| {\nabla_{\phi}}G(\phi_{t},\lambda_{\sigma}^{*}(\phi_{t})) - {\nabla_{\phi}}{G}(\phi_{t},{\lambda^{'}}^{K}_{t},B)||^{2}  \nonumber\\ 
    &+& \frac{1}{\sigma}{\mathbb{E}}||{\nabla_{\phi}}J({\phi_{t}},\lambda^{*}(\phi_{t})) - {\nabla_{\phi}}{J}({\phi_{t}},\lambda^{K}_{t},B)||^{2} \nonumber \\
    &+& \frac{1}{\sigma}{\mathbb{E}}||{\nabla_{\phi}}J(\phi_{t},\lambda_{\sigma}^{*}(\phi_{t}))-  {\nabla_{\phi}}{J}({\phi_{t}},{\lambda^{'}}^{K}_{t},B)||^{2}. \label{main_proof_8}
\end{eqnarray}
As stated in the main text, the error in estimation of the gradient of the pseudo objective is split into the error in estimating ${\nabla_{\phi}}G(\phi,\lambda_{\sigma}^{*}(\phi))$, ${\nabla_{\phi}}J({\phi},\lambda^{*}(\phi))$ and ${\nabla_{\phi}}J({\phi},\lambda_{\sigma}^{*}(\phi))$ whose respective sample based estimates are denoted by ${\nabla_{\phi}}\hat{G}(\phi,{\lambda^{'}}^{K}_{t})$, ${\nabla_{\phi}}\hat{J}({\lambda^{K}_{t},{\phi}})$ and ${\nabla_{\phi}}\hat{J}({{\phi},{\lambda^{'}}^{k}_{t}})$ respectively.
From Lemmas \ref{lemma_2}, \ref{lemma_3}, and \ref{lemma_4}  we have 
\begin{eqnarray}
    \underbrace{{\mathbb{E}}||{\nabla_{\phi}}{\Phi}_{\sigma}(\phi_{t}) - {\nabla_{\phi}}\hat{\Phi}_{\sigma}(\phi_{t}))||^{2}}_{A_{t}}
    &\le&  \tilde{\mathcal{O}} \left(\frac{\gamma^{2H}}{{\sigma}^{2}{B}}\right) + \tilde{\mathcal{O}}\left(\frac{{\exp}^{-K}}{{\sigma}^{2}}\right)  + \tilde{\mathcal{O}}\left(\frac{1}{{\sigma}^{2}{n}}\right) +\tilde{\mathcal{O}}(\epsilon_{approx}) \nonumber\\ 
    && \label{main_proof_9}
\end{eqnarray}
Plugging Equation \eqref{main_proof_9} into Equation \eqref{main_proof_8}, then plugging the result into Equation \eqref{main_proof_4} and squaring both sides we get.
\begin{eqnarray}
  \frac{1}{T}\sum_{i=1}^{T}||{\nabla}\Phi(\phi_{t})||^{2} &\le&   \tilde{\mathcal{O}}\left(\frac{1}{T}\right) + \tilde{\mathcal{O}} \left(\frac{\gamma^{2H}}{{\sigma}^{2}{B}}\right) + \tilde{\mathcal{O}}\left(\frac{{\exp}^{-K}}{{\sigma^{2}}}\right)  + \tilde{\mathcal{O}}\left(\frac{1}{{\sigma}^{2}{n}}\right) +\tilde{\mathcal{O}}(\epsilon_{approx}) \nonumber\\
  &&
\end{eqnarray}

Here $T$ is the number of iterations of the outer loop of Algorithm \ref{algo_1}, $K$ is the number of iterations of the inner loop of Algorithm \ref{algo_1}. $n$ is the number of samples required for the gradients of $J$ with respect to $\lambda$. $B$ is the number of samples used to evaluate the gradients of $G$ with respect to $\lambda$ and $\phi$ respectively and the gradients of $J$ with respect to $\phi$.

\end{proof}
\section{Supplementary Lemmas For Theorem \ref{thm}} \label{supp_lemmas_bi}
\begin{lemma}\label{lemma_2}
For a fixed $\phi_{t} \in \Theta$ and iteration $t$ of Algorithm \ref{algo_1} under Assumptions \ref{assump_1}-\ref{assump_6}  we have 
\begin{eqnarray}
    {\mathbb{E}}||{\nabla}G(\phi_{t},\lambda_{\sigma}^{*}(\phi_{t})) - {\nabla_{\phi}}{G}(\phi,{\lambda'}^{K}_{t},B)||^{2} &\le& \tilde{\mathcal{O}}\left(\frac{\gamma^{2H}}{B}\right) + \tilde{\mathcal{O}}\left({\exp}^{-K}\right)  + \tilde{\mathcal{O}}\left(\frac{1}{n}\right) \nonumber\\
    &+&  \tilde{\mathcal{O}}({\epsilon_{approx}}). \nonumber
\end{eqnarray}
\end{lemma}
\begin{proof}
\begin{eqnarray}
    {\mathbb{E}}|| {\nabla_{\phi}}G(\phi_{t},\lambda_{\sigma}^{*}(\phi_{t})) - {\nabla_{\phi}}{G}(\phi_{t},{\lambda'}^{K}_{t},B)||^{2} &\le&  {\mathbb{E}}|| {\nabla_{\phi}}G(\phi_{t},\lambda_{\sigma}^{*}(\phi_{t})) - {\nabla_{\phi}}G(\phi_{t},{\lambda'}^{K}_{t}) \nonumber\\
    &+& {\nabla_{\phi}}G(\phi,{\lambda'}^{K}_{t}) -   {\nabla_{\phi}}{G}(\phi_{t},{\lambda'}^{K}_{t},B)||^{2}, \label{lem2_1}\\
    &\le&  \underbrace{{\mathbb{E}}|| {\nabla_{\phi}}G(\phi_{t},\lambda_{\sigma}^{*}(\phi_{t})) - {\nabla_{\phi}}G(\phi_{t},{\lambda'}^{K}_{t})||^{2}}_{A^{'}_{K}} \nonumber\\
    &+&  \underbrace{{\mathbb{E}}||{\nabla_{\phi}}G(\phi_{t},{\lambda'}^{K}_{t}) -   {\nabla_{\phi}}{G}(\phi_{t},{\lambda'}^{K}_{t},B)||^{2}}_{B^{'}_{K}}.\label{lem2_2}
\end{eqnarray}
$A^{'}_{K}$ represents the error incurred in due to difference between $\lambda_{\sigma}^{*}(\phi_{t})$ and our estimate ${\lambda'}^{K}_{t}$. $B^{'}_{K}$ represents the difference between the true gradient ${\nabla_{\phi}}G(\phi,{\lambda'}^{K}_{t})$ and its sample-based estimate. 
We first bound $A^{'}_{K}$ as follows
\begin{eqnarray}
 {\mathbb{E}}|| {\nabla_{\phi}}G(\phi_{t},\lambda_{\sigma}^{*}(\phi_{t})) - {\nabla_{\phi}}G(\phi_{t},{\lambda'}^{K}_{t})||^{2}  &\le&  L{\mathbb{E}}||\lambda_{\sigma}^{*}(\phi_{t})-{\lambda'}^{K}_{t})||^{2} \label{lem2_3}\\
                                                                           &\le&  L_{G}{\cdot}{\lambda^{'}}{\mathbb{E}}||h_{\sigma}(\phi_{t},\lambda_{\sigma}^{*}(\phi_{t}))-h_{\sigma}(\phi_{t},{\lambda'}^{K}_{t}))||. \label{lem2_4}
\end{eqnarray}
Here $L_{G}$ is the smoothness constant of $G(\lambda,\phi)$. We get Equation \eqref{lem2_4} from  Equation \eqref{lem2_3} by the quadratic growth property applied to $h_{\sigma}(\phi,\lambda))$ using Assumption \ref{assump_1}. Now, consider the function $h_{\sigma}(\phi,\lambda)$. We know from Assumption \ref{assump_1} that it satisfies the PL condition, therefore using Lemma \ref{thm_supp} we obtain

\begin{eqnarray}
  &&{\mathbb{E}}||h_{\sigma}(\phi_{t},\lambda^{*}(\phi_{t}))-h_{\sigma}(\phi_{t},\lambda^{K}_{t}))||  \le   \tilde{\mathcal{O}}\left({\exp}^{-K}\right) + \mathcal{O}(\beta(n,B,H)),   \label{lem2_temp}
\end{eqnarray}

Where $\beta(n,B,H)$ is such that $\mathbb{E}||{\nabla}_{\lambda}h_{\sigma}(\phi,\lambda) -  {\nabla}_{\lambda}\hat{h}_{\sigma}(\phi,\lambda)||^{2} \le \beta(n,B,H)$. Here, the expectation is with respect to the state action pairs sampled to estimate ${\nabla}_{\lambda}J(\phi,\lambda)$.

Now we have ${\nabla}_{\lambda}\hat{h}_{\sigma}(\phi,\lambda)$ as 
\begin{align}
{\nabla}_{\lambda}\hat{h}_{\sigma}(\phi_{t},\lambda) &= \frac{1}{n}\sum_{i=1}^{n}{\nabla}{\log}(\pi_{\lambda}(a_{i}|s_{i}))\hat{Q_{\phi_{t}}}(s_{i},a_{i}) + \frac{\beta}{B}\sum_{j=1}^{n}\sum_{i=1}^{H} {\gamma}^{i-1}{\nabla}_{\lambda}h_{\pi_{\lambda},\pi_{ref}}(s_{i,j},a_{i,j}) \nonumber\\
&+ \frac{1}{B}\sum_{i=1}^{B}{\nabla}_{\lambda}G(\phi_{t},\lambda)
\end{align}
Thus, in order to bound $\mathbb{E}||{\nabla}_{\lambda}h_{\sigma}(\phi_{t},\lambda) -  {\nabla}_{\lambda}\hat{h}_{\sigma}(\phi_{t},\lambda)||^{2}$, we decompose $\mathbb{E}||{\nabla}_{\lambda}h_{\sigma}(\phi_{t},\lambda) -  {\nabla}_{\lambda}\hat{h}_{\sigma}(\phi_{t},\lambda)||^{2}$ as follows
\begin{eqnarray}
 &&\mathbb{E}||{\nabla}_{\lambda}h_{\sigma}(\phi_{t},\lambda) -  {\nabla}_{\lambda}\hat{h}_{\sigma}(\phi_{t},\lambda)||^{2} \nonumber\\
   &\le&  2\underbrace{(\mathbb{E}||{\nabla}_{\lambda}J(\phi_{t},\lambda) - \frac{1}{n}\sum_{i=1}^{n}{\nabla}{\log}(\pi_{\lambda}(a_{i}|s_{i}))\hat{Q}_{\phi_{t}}(s_{i},a_{i}))||)^{2}}_{A} , \nonumber\\
   &+&  4\underbrace{(\mathbb{E}||{\beta}\sum_{i=1}^{\infty}  \mathbb{E}_{(s_{i},a_{i} \sim \pi_{\lambda})}{\nabla}_{\lambda}h_{\pi_{\lambda},\pi_{ref}}(s^{'}_{i},a^{'}_{i}) - {\beta}\frac{1}{n}\sum_{j=1}^{n}\sum_{i=1}^{H} {\gamma}^{i-1}{\nabla}_{\lambda}h_{\pi_{\lambda},\pi_{ref}}(s_{i,j},a_{i,j})||^{2})}_{B} \nonumber\\
   &+&  {\sigma}4\underbrace{(\mathbb{E}||{\nabla_{\lambda}}G(\phi_{t},{\lambda'}^{K}_{t}) -   {\nabla_{\lambda}}{G}(\phi_{t},{\lambda'}^{K}_{t},B)||)}_{C} \nonumber\\
   &+&  \tilde{\mathcal{O}}\left({\exp}^{-K}\right).\label{lem2_7} 
\end{eqnarray}

Now consider the terms in $A$, if we define $H  =  \mathbb{E}({\nabla}{\log}{\pi_{\lambda}}(a|s)\hat{Q}_{\phi_{t}}(s,a))$ and $d =\frac{1}{n}\sum_{i=1}^{n}{\nabla}{\log}(\pi_{\lambda}(a_{i}|s_{i}))\hat{Q}_{\phi_{t}}(s_{i},a_{i}))$ then we decompose $A$ as follows
\begin{align}
    & \mathbb{E}||{\nabla}J(\phi_{t},\lambda) - d + H -H|| \nonumber\\
    &\le \mathbb{E}||{\nabla}J(\phi_{t},\lambda) - H ||)  +  \mathbb{E}||d -  H||  \label{13_22_1} \\
    &\le  \mathbb{E}||{\nabla}J(\phi_{t},\lambda) -  H||  \nonumber\\
    &+ \mathbb{E}\Bigg|\Bigg|\frac{1}{n}\sum_{i=1}^{n}\left({\nabla}{\log}{\pi_{\lambda_{k}}}(a_{i}|s_{i})\hat{Q}(s_{i},a_{i}) - (H)\right)\Bigg|\Bigg|  \label{13_22_4} \\
    &\le  \mathbb{E}||{\nabla}J(\phi,\lambda) -  H||+ \nonumber\\
    & \mathbb{E}\sqrt{d{\cdot}\sum_{p=1}^{d}\left(\left(\sum_{i=1}^{n}\frac{1}{n}{\nabla}{\log}{\pi_{\lambda}}(a_{i}|s_{i})\hat{Q}_{\phi}(s_{i},a_{i})\right)_{p} - (H)_{p}\right)^{2}}  \label{13_22_4_1} \\
    &\mathbb{E}(||{\nabla}J(\phi,\lambda) - d||) \nonumber\\
    &\le  \mathbb{E}||{\nabla}J(\phi,\lambda) -  H|| + \nonumber\\
    &\sqrt{d{\cdot}\sum_{p=1}^{d} \mathbb{E}\left(\left(\sum_{i=1}^{n}\frac{1}{n}{\nabla}{\log}{\pi_{\lambda_{k}}}(a_{i}|s_{i})\hat{Q}_{\phi}(s_{i},a_{i})\right)_{p} - (H)_{p}\right)^{2}} \label{13_22_5} \nonumber\\
    &\\
    &\le  \mathbb{E}||{\nabla}J(\lambda_{k}) -  H|| +  \frac{1}{\sqrt{n}}{d}{M_{g}}V_{max} \label{13_22_6} \\
    &\le  M_{g}\mathbb{E}_{(s,a)}|Q^{\pi_{\lambda}}_{\phi}(s,a) - \hat{Q}_{\phi}(s,a)| +   \frac{1}{\sqrt{n}}{d}M_{g}V_{max}  \label{13_22_8}    
\end{align}

From \cite{gaur2024closinggapachievingglobal} we have that  
\begin{align}
    \mathbb{E}|Q^{\pi_{\lambda_{k}}}(s,a) - \hat{Q}_{\phi}(s,a)| \le \tilde{O}\left(\frac{1}{\sqrt{n}}\right) + \tilde{\mathcal{O}}(\epsilon_{approx}) \label{13_22_9}
\end{align}

Thus, we obtain 

\begin{align}
 A \le \tilde{O}\left(\frac{1}{n}\right) + \tilde{\mathcal{O}}(\epsilon_{approx})     \label{3_1}
\end{align}

We obtain Equation \eqref{13_22_4_1} from Equation \eqref{13_22_4} by noting that  l1 norm is upper bounded by the l2 norm multiplied by the square root of the dimensions. Here  $({\nabla}{\log}{\pi_{\lambda}}(a_{i}|s_{i})\hat{Q}_{{\phi_{t}}}(s_{i},a_{i}))_{p}$  and $(H)_{p}$ in Equation \eqref{13_22_4_1} are the $p^{th}$ co-ordinates of the gradients. We obtain Equation \eqref{13_22_5} from  Equation \eqref{13_22_4_1} by applying Jensen's inequality on the final term on the right hand side. We obtain Equation \eqref{13_22_6} from  Equation \eqref{13_22_5} by noting that the variance of the random variable ${\nabla}{\log}{\pi_{\lambda_{k}}}(a|s)\hat{Q}(s,a)$ is bounded from Assumption \ref{assump_8} and Assumption \ref{assump_5} which implies that $\Theta$ is a compact set. We combine this with the fact that the variance of the mean is the variance divided by the number of samples, which in this case is $n$.  We obtain Equation \eqref{13_22_8} from Equation \eqref{13_22_6} by using the policy gradient identity which states that ${\nabla}J(\phi,\lambda) = \mathbb{E}{\nabla}log{\pi_{\lambda}}(a|s)Q^{\pi_{\lambda}}_{\phi}(s,a)$  where $M$ is such that $||{\nabla}{\log}{\pi_{\lambda_{k}}}(a|s)|| \le M_{g}$ for all $\lambda \in \Lambda$. We know that $||{\nabla}{\log}{\pi_{\lambda_{k}}}(a|s)||$  are upper bounded by Assumption \ref{assump_8}

We now bound $B$ as follows
\begin{eqnarray}
    &\mathbb{E}||{\beta}\sum_{i=1}^{\infty} {\gamma}^{i-1}\mathbb{E}{\nabla}_{\lambda}h_{\pi_{\lambda},\pi_{ref}}(s_{i},a_{i}) - \frac{{\beta}}{B}\sum_{i=1}^{H}\sum_{j=1}^{B}{\gamma}^{i-1}{\nabla}_{\lambda}h_{\pi_{\lambda},\pi_{ref}}(s^{'}_{i,j},a^{'}_{i,j})|| \label{lem2_8_1}\\
    &\le {\beta}\mathbb{E}||\sum_{i=1}^{H}{\gamma}^{i-1}{\mathbb{E}}{\nabla}_{\lambda}h_{\pi_{\lambda},\pi_{ref}}(s_{i},a_{i})-\frac{1}{B}\sum_{i=1}^{H}\sum_{j=1}^{B}{\gamma}^{i-1}{\nabla}_{\lambda}h_{\pi_{\lambda},\pi_{ref}}(s^{'}_{i,j},a^{'}_{i,j})|| \nonumber\\
    &+{\beta}\mathbb{E}||\sum_{i=H}^{\infty}{\gamma}^{i-1}\mathbb{E}{\nabla}_{\lambda}h_{\pi_{\lambda},\pi_{ref}}(s_{i},a_{i})||  \label{lem2_8_4},\\
      &\le {\beta}\sum_{i=1}^{H}{\gamma}^{i-1}\mathbb{E}||{\nabla}_{\lambda}{\mathbb{E}}h_{\pi_{\lambda},\pi_{ref}}(s_{i},a_{i})-\frac{1}{B}\sum_{j=1}^{B}{\nabla}_{\lambda}h_{\pi_{\lambda},\pi_{ref}}(s^{'}_{i,j},a^{'}_{i,j})|| \nonumber\\
    &+{\beta}\sum_{i=H}^{\infty}{\gamma}^{i-1}\mathbb{E}||{\nabla}_{\lambda}h_{\pi_{\lambda},\pi_{ref}}(s^{'}_{i,j},a^{'}_{i,j})||,   \label{lem2_8_5_1},\\
    &\le \mathcal{O}\left(\frac{\gamma^{H}}{\sqrt{B}}\right) + \mathcal{O}(\gamma^{H}).  \label{lem2_8_6} 
\end{eqnarray}

Note that $(s^{'}_{i,j},a^{'}_{i,j})$ are the sample estimates of $(s_{i},a_{i})$. We obtain Equation \eqref{lem2_8_4} from Equation \eqref{lem2_8_1} by splitting the first term on the left hand side of  Equation \eqref{lem2_8_4} at the point $i=H$. We get Equation \eqref{lem2_8_6} from Equation \eqref{lem2_8_5_1} by considering the fact that the first term on the right hand side  Equation \eqref{lem2_8_5_1} is a variance term bounded by a factor of $\frac{1}{\sqrt{n}}$ since the overall variance of the term $B$ is bounded by Lemma \ref{lem_kl}. The second term on the right hand side of  Equation \eqref{lem2_8_5_1} is bounded since the term ${\nabla}_{\lambda}h_{\pi_{\lambda},\pi_{ref}}(s^{'}_{i,j},a^{'}_{i,j})$ is bounded from Lemma \ref{lem_kl}. Thus, we obtain

\begin{align}
B \le \mathcal{O}\left(\frac{\gamma^{2H}}{n}\right) + \mathcal{O}(\gamma^{H})  \label{3_2}
\end{align}

We now bound $C$ as follows
\begin{eqnarray}
 &&\mathbb{E}||{\nabla_{\lambda}}G({\phi_{t}},{\lambda'}^{K}_{t}) -   {\nabla_{\lambda}}{G}({\phi_{t}},{\lambda'}^{K}_{t},B)|| \nonumber\\
 &=&  \mathbb{E}\sqrt{d{\cdot}\sum_{p=1}^{d} \left(\left({\nabla_{\lambda}}{G}_{{i}}({\phi_{t}},{\lambda'}^{K}_{t})\right)_{p} -    \left(\sum_{i=1}^{B}\frac{1}{B}\mathbb{E}{\nabla_{\lambda}}\hat{G}_{{i}}(\phi,{\lambda'}^{K}_{t})\right)_{p}\right)^{2}}, \nonumber\\
  && \label{lem2_9}\\
  &\le& \sqrt{ \frac{d}{B^{2}}  {\cdot}\sum_{p=1}^{d} {\mathbb{E}}\left(\sum_{i=1}^{B}\left({\nabla_{\lambda}}{G}_{\tau_{i}}({\phi_{t}},{\lambda'}^{K}_{t})_{p} -\mathbb{E}_{\tau}{\nabla_{\lambda}}\hat{G}_(\tau_{i})({\phi_{t}},{\lambda'}^{K}_{t})_{p}\right)\right)^{2}}, \label{lem2_10}\\
  &\le& \sqrt{\frac{d^{2}.B.{\sigma_{G}}}{B^{2}}}, \label{lem2_11}\\
  &\le& \sqrt{d.\frac{{\sigma_{G}}}{B}}, \\
  &\le& \tilde{O}\left(\frac{1}{\sqrt{B}}\right).      \label{lem2_12}
\end{eqnarray}

Here, the right-hand side of Equation \eqref{lem2_9} comes from writing out the definition of the $\ell_{1}$ norm where the subscript of $p$ denotes the $p^{th}$ co-ordinate of the gradient. Equation \eqref{lem2_11} is obtained from Equation \eqref{lem2_10} by using Jensen's Inequality, and Equation \eqref{lem2_12} is obtained from \ref{lem2_11} using Assumption \ref{assump_6} which states that the variance of ${\nabla}G$ estimator is bounded.

This gives us 

\begin{align}
C \le \mathcal{O}\left(\frac{1}{B}\right)     \label{3_3}
\end{align}

Combining Equation \eqref{lem2_8_6} and \eqref{13_22_9} we have that 

\begin{align}
    \mathbb{E}||{\nabla}_{\lambda}h_{\sigma}({\phi_{t}},\lambda) -  {\nabla}_{\lambda}\hat{h}_{\sigma}({\phi_{t}},\lambda)||^{2}  &\le  \mathcal{O}\left(\frac{\gamma^{2H}}{{B}}\right) +  \mathcal{O}\left(\frac{1}{{B}}\right)  + \mathcal{O}(\gamma^{H}) + \tilde{O}\left(\frac{1}{{n}}\right) + \tilde{\mathcal{O}}(\epsilon_{approx}) \nonumber\\
     &\le  \mathcal{O}\left(\frac{\gamma^{2H}}{{B}}\right) + \tilde{\mathcal{O}}\left(\frac{1}{{n}}\right) + \tilde{\mathcal{O}}(\epsilon_{approx})
\end{align}

Which in turn gives us 

\begin{eqnarray}
  &&{\mathbb{E}}||h_{\sigma}(\phi_{t},\lambda^{K}_{t}))- h_{\sigma}(\phi_{t},\lambda^{*}(\phi_{t}))||  \le   \tilde{\mathcal{O}}\left(\frac{1}{{n}}\right) + \tilde{\mathcal{O}}\left({\exp}^{-K}\right) + \mathcal{O}\left(\frac{\gamma^{2H}}{{B}}\right)  + \tilde{\mathcal{O}}(\epsilon_{approx}),  \label{lem2_temp_3}
\end{eqnarray}

We can bound $B_{K}^{'}$ in the exact same manner as $C$ where the gradient is with respect to $\lambda$ instead of $\phi$ to get  

\begin{align}
B_{K}^{'} \le \mathcal{O}\left(\frac{1}{B}\right)     \label{3_4}
\end{align}

Thus we obtain 
\begin{align}
    {\mathbb{E}}||{\nabla_{\phi}}G({\phi_{t}},\lambda^{K}_{t}) -   {\nabla_{\phi}}{G}(\phi,\lambda^{K}_{t},B)||^{2} \le \tilde{O}\left(\frac{1}{{B}}\right) \label{lem2_13}
\end{align}

Substituting Equation \eqref{lem2_temp_3} into Equation \eqref{lem2_4}. Then put the result from Equation \eqref{lem2_4} and Equation \eqref{lem2_13} in Equation \eqref{lem2_2} to get the required result.

\end{proof}

\begin{lemma}\label{lemma_3} 
For a fixed $\phi_{t} \in \Theta$ and iteration $t$ of Algorithm \ref{algo_1} under Assumptions \ref{assump_1}-\ref{assump_6}  we have 
\begin{eqnarray}
    {\mathbb{E}}||{\nabla_{\phi}}J(\phi_{t},{\lambda^{*}}(\phi_{t}))-  {\nabla_{\phi}}{J}({{\phi_{t}},{\lambda}^{K}_{t}(\phi)},B)||^{2} &\le&   \tilde{\mathcal{O}}\left(\frac{\gamma^{2H}}{B}\right)  +  
 \tilde{\mathcal{O}}\left( {\exp}^{-K}\right)  + \tilde{\mathcal{O}}\left(\frac{1}{n}\right)  \nonumber\\
    &+& \tilde{\mathcal{O}}({\epsilon_{approx}})
\end{eqnarray}
\end{lemma}
\begin{proof}
\begin{eqnarray}
        && {\mathbb{E}}||{\nabla_{\phi}}J(\phi_{t},{\lambda}^{*}(\phi_{t}))-  {\nabla_{\phi}}{J}({{\phi_{t}},{\lambda}^{K}_{t}(\phi)},B)||^{2} \nonumber\\
        &\le&  {\mathbb{E}}||{\nabla_{\phi}}J(\phi_{t},{\lambda}^{*}(\phi_{t}))- {\nabla_{\phi}}J(\phi_{t},{\lambda}^{K}_{t}) + {\nabla_{\phi}}J(\phi_{t},{\lambda}^{K}_{t})-{\nabla_{\phi}}{J}({\phi_{t},{\lambda}^{K}_{t}(\phi)},B)||^{2}, \label{lem3_a}\\
        &\le&  {\mathbb{E}}||{\nabla_{\phi}}J(\phi_{t},{\lambda}^{*}(\phi_{t}))- {\nabla_{\phi}}J(\phi_{t},{\lambda}^{K}_{t})||^{2} + {\mathbb{E}}||{\nabla_{\phi}}J(\phi_{t},{\lambda}^{K}_{t})  {\nabla_{\phi}}{J}({\phi_{t},{\lambda}^{K}_{t}(\phi)},B)||^{2}, \label{lem3_b}\\
        &\le& L{\mathbb{E}}||({\lambda^{*}(\phi_{t})})- ({\lambda}^{K}_{t})||^{2} + {\mathbb{E}}||{\nabla_{\phi}}J(\phi_{t},{\lambda}^{K}_{t})-  {\nabla_{\phi}}{J}({\phi_{t},{\lambda}^{K}_{t}(\phi)},B)||^{2}, \label{lem3_0}\\
        &\le&  \underbrace{{\mu}{\cdot}L{\mathbb{E}}||J(\phi_{t},{\lambda^{*}}(\phi_{t}))- J(\phi_{t},{\lambda}^{K}_{t})||^{2}}_{A_{K}^{''}} + \underbrace{{\mathbb{E}}||{\nabla_{\phi}}J(\phi_{t},{\lambda}^{K}_{t})-  {\nabla_{\phi}}{J}({\phi_{t},{\lambda}^{K}_{t}(\phi)},B)||^{2}}_{B_{K}^{''}}. \label{lem3_1}
\end{eqnarray}

We get Equation \eqref{lem3_0} from Equation \eqref{lem3_b} by the smoothness of $J(\phi,\lambda)$ using Assumption \ref{assump_1}. We get Equation \eqref{lem3_1} from \eqref{lem3_0} by the quadratic growth inequality on $J(\phi,\lambda)$. The first term $A_{K}^{''}$ is upper bounded using the same way as is done for  $A_{K}^{'}$  Lemma \ref{lemma_2}, with the only difference being the absence of the term $C$ in Equation \eqref{lem2_7}. Thus, we have

\begin{eqnarray}
  {\mathbb{E}}||J(\phi_{t},\lambda^{*}(\phi_{t}))-J(\phi_{t},\lambda^{K}_{t}))||  &\le& 
   \mathcal{\tilde{O}}\left(\exp^{-K}\right)  + {\tilde{\mathcal{O}}}\left(\frac{1}{{n}}\right)     
   + \mathcal{O}\left(\frac{\gamma^{2H}}{{B}}\right)  + \tilde{\mathcal{O}}({\epsilon_{approx}}).  \label{lem3_1_1} 
\end{eqnarray}
We bound $B_{K}^{''}$ as follows
\begin{eqnarray}
&&    \mathbb{E}||{\nabla_{\phi}}J(\phi_{t},{\lambda}^{K}_{t})-  {\nabla_{\phi}}{J}({\phi_{t},{\lambda}^{K}_{t}(\phi)},B)|| \nonumber\\
&=&
    \mathbb{E}\Bigg|\Bigg|\sum_{i=1}^{\infty} {\gamma}^{i-1} \mathbb{E}[{\nabla}_{\phi}r_{\phi_{t}}(s_{i},a_{i})] -   \frac{1}{B}\sum^{B}_{j=1}\sum_{i=1}^{H}{\gamma}^{i-1}{\nabla}_{\phi}r_{\phi_{t}}(s^{'}_{i,j},a^{'}_{i,j})\Bigg|\Bigg| \nonumber\\
    &\le&  \sum_{i=1}^{H} {\gamma}^{i-1} \Big(\mathbb{E}||\mathbb{E}[{\nabla}_{\phi}r_{\phi_{t}}(s_{i},a_{i})] -\frac{1}{B}\sum_{j=1}^{B}{\nabla}_{\phi}r_{\phi_{t}}(s_{i,j},a_{i,j})||\Big) \nonumber\\
    &+& \Bigg|\Bigg|\sum_{i=H}^{\infty}{\gamma}^{i-1} \mathbb{E}[{\nabla}_{\phi}r_{\phi_{t}}(s_{i},a_{i})]\Bigg|\Bigg|,  \label{lem3_3}\\
    &\le&  \tilde{\mathcal{O}}\left({\frac{\gamma^{H}}{\sqrt{B}}}\right) + \tilde{\mathcal{O}}({\gamma}^{H}).   \label{lem3_4}
\end{eqnarray}

Thus we have 

\begin{eqnarray}
 \mathbb{E}||{\nabla_{\phi}}J(\phi_{t},{\lambda}^{K}_{t})-  {\nabla_{\phi}}{J}({\phi_{t},{\lambda}^{K}_{t}},B)||^{2} 
    \le  \tilde{\mathcal{O}}\left({\frac{\gamma^{2H}}{{B}}}\right) + \tilde{\mathcal{O}}({\gamma}^{H}).   \label{lem3_5}
\end{eqnarray}

We get Equation \eqref{lem3_4} from Equation \eqref{lem3_3} since the first term on the right hand side of Equation \eqref{lem3_3} is variance term with a sample size of $B$. The last term  on the right hand side of Equation \eqref{lem3_3} is upper bounded by $\gamma^{H}$ since the term ${\nabla}_{\phi}r_{\phi}(s_{i},a_{i})$ is upper bounded by Assumption \ref{assump_8}.

Plugging the result of Equation \eqref{lem3_5} and Equation \eqref{lem3_1_1} into Equation \eqref{lem3_1} gives us the required result..

\end{proof}
\begin{lemma} \label{lemma_4}
For a fixed $\phi_{t} \in \Theta$ and iteration $t$ of Algorithm \ref{algo_1} under Assumptions \ref{assump_1}-\ref{assump_6}  we have 
\begin{eqnarray}
    {\mathbb{E}}||{\nabla_{\phi}}J(\phi_{t},{\lambda_{\sigma}^{*}}(\phi_{t}))-  {\nabla_{\phi}}{J}({\phi_{t},{\lambda^{'}}^{K}_{t}(\phi)},B)||^{2} &\le&   
 \tilde{\mathcal{O}}\left(\frac{{\gamma^{2H}}}{{{B}}}\right) +
 \tilde{\mathcal{O}}\left({\exp}^{-K}\right)  + \tilde{\mathcal{O}}\left(\frac{1}{{n}}\right)\nonumber\\
    &+& \tilde{\mathcal{O}}({\epsilon_{approx}})
\end{eqnarray}
\end{lemma}
\begin{proof}
\begin{eqnarray}
       && {\mathbb{E}}||{\nabla_{\phi}}J(\phi_{t},{\lambda_{\sigma}^{*}}(\phi_{t}))-  {\nabla_{\phi}}{J}({\phi_{t},{\lambda^{'}}^{K}_{t}(\phi)},B)|| \nonumber\\
       &\le&  {\mathbb{E}}||{\nabla_{\phi}}J(\phi_{t},{\lambda_{\sigma}^{*}}(\phi_{t}))- {\nabla_{\phi}}J(\phi_{t},{\lambda^{'}}^{k}_{t}) + {\nabla_{\phi}}J(\phi_{t},{\lambda^{'}}^{k}_{t})- {\nabla_{\phi}}{J}({\phi_{t},{\lambda^{'}}^{K}_{t}(\phi)},B)||^{2},\\
        &\le&  {\mathbb{E}}||{\nabla_{\phi}}J(\phi_{t},{\lambda_{\sigma}^{*}}(\phi_{t}))- {\nabla_{\phi}}J(\phi_{t},{\lambda^{'}}^{k}_{t})||^{2} + ||{\nabla_{\phi}}J(\phi_{t},{\lambda^{'}}^{k}_{t})-{\nabla_{\phi}}{J}({\phi_{t},{\lambda^{'}}^{K}_{t}(\phi)},B)||^{2},\\
        &\le& L_{J}.{\mathbb{E}}||({\lambda_{\sigma}^{*}(\phi_{t})})- ({\lambda^{'}}^{K}_{t})||^{2} +{\mathbb{E}} ||{\nabla_{\phi}}J(\phi_{t},{\lambda^{'}}^{k}_{t})-  {\nabla_{\phi}}{J}({\phi_{t},{\lambda^{'}}^{K}_{t}(\phi)},B)||^{2}, \label{lem4_0}\\
        &\le&  \underbrace{L_{J}.\mu{\mathbb{E}}||h_{\sigma}(\phi_{t},{\lambda_{\sigma}^{*}}(\phi_{t}))- h_{\sigma}(\phi,{\lambda^{'}}^{K}_{t})||}_{A_{k}^{'''}} + \underbrace{{\mathbb{E}}||{\nabla_{\phi}}J(\phi_{t},{\lambda^{'}}^{K}_{t})-  {\nabla_{\phi}}{J}({\phi_{t},{\lambda^{'}}^{K}_{t}},B)||^{2}}_{B_{k}^{'''}}.\label{lem4_1}
\end{eqnarray}
We get Equation \eqref{lem4_1} from Equation \eqref{lem4_0} using Assumption \ref{assump_1}. Note that $B_{k}^{'''}$ here is the same as $B''_{K}$ in Lemma \ref{lemma_3}. Thus we have 
\begin{eqnarray}
     {\mathbb{E}}||{\nabla_{\phi}}J(\phi_{t},{\lambda^{'}}^{K}_{t})-  {\nabla_{\phi}}\hat{J}({\phi_{t},{\lambda^{'}}^{K}_{t}(\phi)})||^{2} &\le& \tilde{\mathcal{O}}\left({\frac{\gamma^{2H}}{B}}\right) + \tilde{\mathcal{O}}({\gamma}^{H})  \label{lem4_2}
\end{eqnarray}
Further, we have

\begin{eqnarray}
 {\mathbb{E}}||h_{\sigma}(\phi_{t},{\lambda_{\sigma}^{*}}(\phi_{t}))- h_{\sigma}(\phi_{t},{\lambda^{'}}^{K}_{t})|| &\le&   \tilde{\mathcal{O}}\left(\frac{1}{{n}}\right) + \tilde{\mathcal{O}}\left({\exp}^{-K}\right) + \mathcal{O}\left(\frac{\gamma^{2H}}{{B}}\right)  + \tilde{\mathcal{O}}(\epsilon_{approx}), \label{lem4_3}
\end{eqnarray}

This is the same result as for $A'_K$ in  Lemma \ref{lemma_2}.

Plugging  Equations \eqref{lem4_2} and \eqref{lem4_3} into Equation \eqref{lem4_1} given us the required result.
\end{proof}

\begin{lemma}\label{lemma_5} For a given $\lambda \in \Lambda$ and $\phi \in \Theta$ we have 
\begin{eqnarray}
    \nabla_{\phi}J(\phi,\lambda)  &=&  \sum_{i=1}^{\infty} {\gamma}^{i-1} \mathbb{E}{\nabla}_{\phi}r_{\phi}(s_{i},a_{i})
\end{eqnarray}
\end{lemma}

\begin{proof}
We start by writing the gradient of $J(\phi,\lambda)$ with respect to $\phi$ as follows 
\begin{eqnarray}
  &&  \nabla_{\phi}J(\phi,\lambda)\nonumber\\
  &=&   \nabla_{\phi}\int_{s_{1},a_{1}} Q^{\lambda}_{\phi}(s_{1},a_{1})\pi_{\lambda}(a_1|s_1)d{(s_{1})}  \label{lem5_0}\\
                                     &=&   \int_{s_{1},a_{1}} \nabla_{\phi}r_{\phi}(s_{1},a_{1})\pi_{\lambda}(a_1|s_1)d{(s_{1})} \nonumber\\
                                     &+& {\gamma}{\cdot} \nabla_{\phi}\int_{s_{1},a_{1}}\int_{s_{2},a_{2}} Q^{\lambda}_{\phi}(s_{2},a_{2})d{(s_{2}|a_{1})}\pi_{\lambda}(a_2|s_2)d{(s_{1})}\pi_{\lambda}(a_1|s_1),\label{lem5_1}\\
                                     &=&   \int_{s_{1},a_{1}}\nabla_{\phi}r_{\phi}(s_{1},a_{1})\pi_{\lambda}(a_1|s_1)d{(s_{1})} \nonumber\\
                                     &+& {\gamma}{\cdot}\int_{s_{2},a_{2}}\int_{s_{1},a_{1}} \nabla_{\phi}r_{\phi}(s_{2},a_{2})d{(s_{2}|a_{1})}\pi_{\lambda}(a_2|s_2)d{(s_{1})}\pi_{\lambda}(a_1|s_1)\nonumber\\
                                    &+&  {\gamma}^{2}{\cdot} \nabla_{\phi}\int_{s_{1},a_{1}} \int_{s_{2},a_{2}} \int_{s_{3},a_{3}} Q^{\lambda}_{\phi}(s_{3},a_{3}) 
 d(a_{3}|s_{3})d(s_{3}|a_{2})d{(s_{2}|a_{1})}\pi_{\lambda}(a_2|s_2)d{(s_{1})}\pi_{\lambda}(a_1|s_1),\nonumber\\
                                    && \label{lem5_2}\\
                                      &=&   \int_{s_{1},a_{1}}\nabla_{\phi}r_{\phi}(s_{1},a_{1})d(s_{1},a_{1}) \nonumber\\
                                      &+& {\gamma}{\cdot}
                                     \int_{s_{2},a_{2}}\nabla_{\phi}r_{\phi}(s_{2},a_{2})d(s_{2},a_{3}) + {\gamma}^{2}{\cdot} \nabla_{\phi}\int_{s_{3},a_{3}} Q^{\lambda}_{\phi}(s_{3},a_{3})d(s_{3},a_{3}).\nonumber\\
                                     && \label{lem5_3}
\end{eqnarray}
We get Equation \eqref{lem5_1} from  Equation \eqref{lem5_0} by noting that $Q^{\lambda}_{\phi}(s,a) = r_{\phi} + \int_{s^{'},a^{'}}Q^{\lambda}_{\phi}(s^{'},a^{'})d(s^{'}|a)\pi_{\lambda}(a^{'}|s^{'})$. We repeat the same process on the second term on the right hand side of Equation \eqref{lem5_1}  to obtain Equation \eqref{lem5_2}.
Continuing this sequence, we get
\begin{eqnarray}
    \nabla_{\phi}J^{\lambda}_{\phi}  &=&  \sum_{i=1}^{\infty} {\gamma}^{i-1} \mathbb{E}{\nabla}_{\phi}r_{\phi}(s_{i},a_{i})
\end{eqnarray}
Here, $s_{i},a_{i}$ belong to the distribution of the $i^{th}$ state action pair induced by following the policy $\lambda$.
 \end{proof}

\section{Proof of Theorem \ref{thm_bi}} \label{thm_bi_proof}
\begin{proof}
As is done for the proof of Theorem \ref{thm} we obtain the following from the smoothness assumption on $\Phi$.
\begin{eqnarray}
  \frac{1}{T}\sum_{i=1}^{T}||{\nabla}\Phi(\phi_{t})||^{2}  &\le&   \frac{1}{T}\sum^{t=T}_{k=0}\underbrace{{\mathbb{E}}||{\nabla_{\phi}}{\Phi}(\phi_{t}) - {\nabla_{\phi}}\hat{\Phi}_{\sigma}( \phi_{t})||^{2}}_{A_{t}} + \tilde{\mathcal{O}}\left(\frac{1}{T}\right). \label{proof_bi_1}
\end{eqnarray}

We now bound $A_{t}$ as follows
\begin{eqnarray}
    {\mathbb{E}}||{\nabla_{\phi}}{\Phi}(\phi_{t}) - {\nabla_{\phi}}\hat{\Phi}_{\sigma}(\phi_{t}))||^{2} &=& {\mathbb{E}}||{\nabla_{\phi}}{\Phi}(\phi_{t}) - {\nabla_{\phi}}{\Phi}_{\sigma}(\phi_{t}) + {\nabla_{\phi}}{\Phi}_{\sigma}(\phi_{t}) - {\nabla_{\phi}}\hat{\Phi}_{\sigma}(\phi_{t}))||^{2}, \nonumber\\
    &&  \label{proof_bi_2}\\
    &\le& {\mathbb{E}}||{\nabla_{\phi}}{\Phi}(\phi_{t}) - {\nabla_{\phi}}{\Phi}_{\sigma}(\phi_{t}))||^{2} \nonumber\\
    &+& {\mathbb{E}}||{\nabla_{\phi}}{\Phi}_{\sigma}(\phi_{t}) - {\nabla_{\phi}}\hat{\Phi}_{\sigma}(\phi_{t}))||^{2},  \label{main_proof_34}\\
    &\le&  \mathcal{O}(\sigma) +  \underbrace{{\mathbb{E}}||{\nabla_{\phi}}{\Phi}_{\sigma}(\phi_{t}) - {\nabla_{\phi}}\hat{\Phi}_{\sigma}(\phi_{t}))||^{2}}_{A_{t}^{'}}, \label{proof_bi_3}
\end{eqnarray}
The first term on the right hand side denotes the gap between the gradient of the objective function and the gradient of the pseudo-objective $\Phi_{\sigma}$. We get the upper bound on this term form \cite{chen2024findingsmallhypergradientsbilevel}. The term $A^{'}_{t}$ denotes the error incurred in estimating the true gradient of the pseudo-objective. 
\begin{eqnarray}
    \underbrace{{\mathbb{E}}||{\nabla_{\phi}}{\Phi}_{\sigma}(\phi_{t}) - {\nabla_{\phi}}\hat{\Phi}_{\sigma}(\phi_{t})||^{2}}_{A_{t}^{'}}
    &\le&  {\mathbb{E}}\Bigg|\Bigg| {\nabla_{\phi}}G(\phi,\lambda_{\sigma}^{*}(\phi))  +  \frac{{\nabla_{\phi}}J({\lambda^{*}(\phi),{\phi}}) -  {\nabla_{\phi}}J(\phi,{\lambda_{\sigma}^{*}}(\phi))}{\sigma} \nonumber\\
    &-& {\nabla_{\phi}}{G}(\phi_{t},{\lambda'}^{K}_{t},B)  +  \frac{{\nabla_{\phi}}{J}(\phi_{t},{\lambda^{K}_{t}(\phi_{t})},B) -  {\nabla_{\phi}}{J}({\phi_{t},{\lambda^{'}}^{K}_{t}(\phi)},B)}{\sigma}\Bigg|\Bigg|^{2},\nonumber\\ 
    \label{proof_bi_4}\\
    &\le& {\mathbb{E}}|| {\nabla_{\phi}}G(\phi_{t},\lambda_{\sigma}^{*}(\phi_{t})) - {\nabla_{\phi}}{G}(\phi_{t},{\lambda'}^{K}_{t},B)||^{2}  \nonumber\\ 
    &+& \frac{1}{\sigma}{\mathbb{E}}||{\nabla_{\phi}}J(\phi_{t},{\lambda^{*}(\phi_{t})}) - {\nabla_{\phi}}{J}(\phi_{t},{\lambda^{K}_{t}},B)||^{2} \nonumber \\
    &+& \frac{1}{\sigma}{\mathbb{E}}||{\nabla_{\phi}}J({\phi_t},{\lambda_{\sigma}^{*}}({\phi_t}))-  {\nabla_{\phi}}{J}({{\phi_t},{\lambda^{'}}^{k}_{t}},B)||^{2}. \label{proof_bi_5}
\end{eqnarray}
As stated in the main text, the error in estimation of the gradient of the pseudo objective is split into the error in estimating ${\nabla_{\phi}}G({\phi_t},\lambda_{\sigma}^{*}({\phi_t}))$, ${\nabla_{\phi}}J({\phi_t},\lambda^{*}({\phi_t}))$ and ${\nabla_{\phi}}J({\phi_t},{\lambda_{\sigma}^{*}}({\phi_t}))$ whose respective sample based estimates are denoted by ${\nabla_{\phi}}\hat{G}({\phi_t},{\lambda'}^{K}_{t})$, ${\nabla_{\phi}}\hat{J}({\phi_t}\lambda^{K}_{t})$ and ${\nabla_{\phi}}\hat{J}({{\phi_t},{\lambda^{'}}^{K}_{t}})$ respectively.
From Lemmas \ref{lemma_6}, \ref{lemma_7}, and \ref{lemma_8}  we have 
\begin{eqnarray}
    \underbrace{{\mathbb{E}}||{\nabla_{\phi}}{\Phi}_{\sigma}({\phi_t}) - {\nabla_{\phi}}\hat{\Phi}_{\sigma}({\phi_t}))||^{2}}_{A_{t}^{'}}
    &\le&  \tilde{\mathcal{O}} \left(\frac{1}{{\sigma^{2}}{B}}\right) + \tilde{\mathcal{O}}\left(\frac{\exp^{-K}}{{\sigma^{2}}}\right) \label{proof_bi_6}
\end{eqnarray}
Plugging Equation \eqref{proof_bi_6} into Equation \eqref{proof_bi_5}, then plugging the result into Equation \eqref{proof_bi_1} we get
\begin{eqnarray}
  \frac{1}{T}\sum_{i=1}^{T}||{\nabla}\Phi(\phi_{t})||^{2} &\le&   \tilde{\mathcal{O}}\left(\frac{1}{T}\right) +   \tilde{\mathcal{O}}\left(\frac{\exp^{-K}}{{\sigma}^{2}}\right)  + \tilde{\mathcal{O}}\left(\frac{1}{{\sigma}^{2}{B}}\right)  + \tilde{\mathcal{O}}(\sigma^{2}) 
\end{eqnarray}
Here $T$ is the number of iterations of the outer loop of Algorithm \ref{algo_1}, $K$ is the number of iterations of the inner loop of Algorithm \ref{algo_1}. $B$ is the number of samples required for the all the gradient evaluations.
\end{proof}
\section{Supplementary Lemmas For Theorem \ref{thm_bi}}  \label{supp_lemmas}
\begin{lemma}\label{lemma_6}
For a fixed ${\phi_t} \in \Theta$ and iteration $t$ of Algorithm \ref{algo_1} under Assumptions \ref{assump_1}-\ref{assump_5} and Assumptions \ref{assump_7}  we have 
\begin{eqnarray}
    {\mathbb{E}}|| {\nabla}G({\phi_t},\lambda^{*}({\phi_t})) - {\nabla_{\phi}}{G}({\phi_t},\lambda^{K}_{t},B)||^{2} &\le& \tilde{\mathcal{O}}\left(\frac{1}{{B}}\right) + \tilde{\mathcal{O}}\left({\exp}^{-K}\right) 
\end{eqnarray}
\end{lemma}
\begin{proof}
\begin{eqnarray}
    {\mathbb{E}}|| {\nabla_{\phi}}G({\phi_t},\lambda_{\sigma}^{*}({\phi_t})) - {\nabla_{\phi}}{G}({\phi_t},{\lambda'}^{K}_{t},B)||^{2} &\le&  {\mathbb{E}}|| {\nabla_{\phi}}G({\phi_t},\lambda_{\sigma}^{*}(\phi)) - {\nabla_{\phi}}G({\phi_t},{\lambda'}^{K}_{t}) \nonumber\\
    &+& {\nabla_{\phi}}G({\phi_t},{\lambda'}^{K}_{t}) -   {\nabla_{\phi}}{G}({\phi_t},{\lambda'}^{K}_{t},B)||^{2}, \label{lem6_1}\\
    &\le&  2\underbrace{{\mathbb{E}}|| {\nabla_{\phi}}G({\phi_t},\lambda_{\sigma}^{*}({\phi_t})) - {\nabla_{\phi}}G({\phi_t},{\lambda'}^{K}_{t})||^{2}}_{A^{'}_{K}} \nonumber\\
    &+&  2\underbrace{{\mathbb{E}}||{\nabla_{\phi}}G({\phi_t},{\lambda'}^{K}_{t}) -   {\nabla_{\phi}}{G}({\phi_t},{\lambda'}^{K}_{t},B)||^{2}}_{B^{'}_{K}}.\label{lem6_2}
\end{eqnarray}

We first bound  $A^{'}_{K}$. 
\begin{eqnarray}
 {\mathbb{E}}||{\nabla_{\phi}}G({\phi_t},\lambda_{\sigma}^{*}(\phi)) - {\nabla_{\phi}}G({\phi_t},{\lambda'}^{K}_{t})||^{2}  &\le&  L{\mathbb{E}}||\lambda_{\sigma}^{*}({\phi_t})-{\lambda'}^{K}_{t})||^{2} \label{lem6_3}\\
                                                                           &\le&  L_{1}{\cdot}{\mu}{\mathbb{E}}||h_{\sigma}({\phi_t},{\lambda_{\sigma}^{*}}({\phi_t}))- h_{\sigma}({\phi_t},{\lambda^{'}}^{k}_{t})||. \label{lem6_4}
\end{eqnarray}
Here $L_{1}$ is the smoothness constant of $G(\lambda,\phi)$. We get Equation \eqref{lem6_4} from  Equation \eqref{lem6_3} by Assumption \ref{assump_1}. Now, consider the function $J(\phi,\lambda)$. We know from Lemma \ref{thm_supp} that it satisfies the weak gradient condition, therefore, applying the same logic for $J(\phi,\lambda)$ that we did for $\Phi(\sigma)$. Using Assumption \ref{assump_1}, and Lemma \ref{thm_supp} we obtain 

\begin{eqnarray}
  {\mathbb{E}}||h_{\sigma}({\phi_t},{\lambda_{\sigma}^{*}}(\phi_{t}))- h_{\sigma}(\phi_{t},{\lambda^{'}}^{K}_{t})|| \le   \beta(B) + \tilde{\mathcal{O}}\left({\exp}^{-K}\right),  \label{lem6_6}
\end{eqnarray}

where $\beta(n,B,H)$ satisfies $\mathbb{E}||\nabla_{\lambda}{h}_{\sigma}({\phi_t},\lambda)-{\nabla}{h}_{\sigma}(\phi_{t},\lambda))||^{2} \le \delta(B)$. Note we changed notation from $\beta(n,B,H)$ to $\beta(B)$ since $B$ samples are used to evaluate the gradients. Now in this case, we have an unbiased estimate of $\nabla{h}_{\sigma}({\phi_t},\lambda^{*}({\phi_t}))$. Therefore, from assumption \ref{assump_7} we have that.

Now, the term $\mathbb{E}||{\nabla}h_{\sigma}({\phi_t},{\lambda}) - {\nabla}\hat{h}_{\sigma}(\phi,{\lambda})||^{2}$, it can be decomposed as follows
\begin{eqnarray}
&&\mathbb{E}||\nabla_{\lambda}h_{\sigma}(\phi_{t},{\lambda}) - \nabla_{\lambda}\hat{h}_{\sigma}(\phi_{t},{\lambda})||^{2} \nonumber\\
&=& \mathbb{E}||{\nabla}_{\lambda}J(\phi_{t},{\lambda}) + {\sigma}{\nabla}_{\lambda}G(\phi_{t},{\lambda}) - {\nabla}_{\lambda}{J}(\phi_{t},\lambda,B) - {\sigma}{\nabla}_{\lambda}G(\phi_{t},{\lambda},B)||^{2}, \label{lem6_7}\\
 &\le&  \underbrace{\mathbb{E}||{\nabla}_{\lambda}J(\phi_{t},{\lambda})-{\nabla}_{\lambda}{J}(\phi_{t},{\lambda},B)||^{2}}_{A^{'''}} + {\sigma}\underbrace{\mathbb{E}||{\nabla}_{\lambda}G(\phi_{t},{\lambda})-{\nabla}_{\lambda}{G}(\phi_{t},{\lambda},B)||^{2}}_{B^{'''}}. \label{lem6_8}
\end{eqnarray}

Note that both $A^{'''}$ and $B^{'''}$ can be bounded same as $C$ in Lemma \ref{lemma_2}. thus we have 

\begin{eqnarray}
A^{'''} \le \tilde{\mathcal{O}}\left(\frac{1}{{B}}\right) \label{lem6_9}\\
B^{'''} \le \tilde{\mathcal{O}}\left(\frac{1}{{B}}\right) \label{lem6_10}
\end{eqnarray}

Thus we have $\beta(B) = \tilde{\mathcal{O}}\left(\frac{1}{B}\right)$. Which gives us
\begin{eqnarray}
 {\mathbb{E}}||h_{\sigma}(\phi_{t},{\lambda_{\sigma}^{*}})- h_{\sigma}(\phi_{t},{\lambda^{'}}^{K}_{t})||  &\le&   \tilde{\mathcal{O}}\left({\exp}^{-K}\right)  + \tilde{\mathcal{O}}\left(\frac{1}{B}\right) \label{lem6_11}
\end{eqnarray}

Similarly  $B_{k}^{'}$ here is bounded the same  way as $C$ in Lemma \ref{lemma_2} to get 

\begin{eqnarray}
    {\mathbb{E}}||{\nabla_{\phi}}G(\phi_{t},\lambda^{K}_{t}) -   {\nabla_{\phi}}{G}(\phi_{t},\lambda^{K}_{t},B)||^{2} \le \mathcal{O}\left(\frac{1}{{B}}\right)  \label{lem6_12} 
\end{eqnarray}

Plugging Equation \eqref{lem6_11} and \eqref{lem6_12} into Equation \eqref{lem6_2} gives us the required result.
\end{proof}

\begin{lemma}\label{lemma_7} 
For a fixed $\phi_{t} \in \Theta$ and iteration $t$ of Algorithm \ref{algo_1} under Assumptions \ref{assump_1}-\ref{assump_5} and Assumptions \ref{assump_7}  we have 
\begin{eqnarray}
    {\mathbb{E}}||{\nabla_{\phi}}J(\phi_{t},{\lambda^{*}}(\phi))-  {\nabla_{\phi}}{J}({\phi_{t},{\lambda}^{K}_{t}(\phi)},B)||^{2} &\le&   \tilde{\mathcal{O}}\left(\frac{1}{{B}}\right)  +  
 \tilde{\mathcal{O}}\left({\exp}^{-K}\right) 
\end{eqnarray}
\end{lemma}
\begin{proof}
\begin{eqnarray}
        && {\mathbb{E}}||{\nabla_{\phi}}J(\phi_{t},{\lambda^{*}}(\phi_{t}))-  {\nabla_{\phi}}{J}({\phi_{t},{\lambda}^{K}_{t}},B)||^{2} \nonumber\\
        &\le&  {\mathbb{E}}||{\nabla_{\phi}}J(\phi_{t},{\lambda^{*}}(\phi_{t}))- {\nabla_{\phi}}J(\phi_{t},{\lambda}^{K}_{t}) + {\nabla_{\phi}}J(\phi,{\lambda}^{K}_{t}) -{J}({{\phi},{\lambda}^{K}_{t}(\phi)},B)||^{2}, \label{lem7_a}\\
        &\le&  {\mathbb{E}}||{\nabla_{\phi}}J(\phi_{t},{\lambda^{*}}(\phi_{t}))- {\nabla_{\phi}}J(\phi_{t},{\lambda}^{K}_{t})||^{2} + ||{\nabla_{\phi}}J(\phi_{t},{\lambda}^{K}_{t}) - {\nabla_{\phi}}{J}({{\phi},{\lambda}^{K}_{t}},B)||^{2}, \label{lem7_b}\\
        &\le& L^{'}{\mathbb{E}}||({\lambda^{*}(\phi_{t})})- ({\lambda}^{K}_{t})||^{2} + ||{\nabla_{\phi}}J(\phi_{t},{\lambda}^{K}_{t})-  {\nabla_{\phi}}{J}({{\phi_t},{\lambda}^{K}_{t}(\phi)},B)||^{2}, \label{lem7_0}\\
        &\le&  \underbrace{L^{'}{\cdot}{\mu}{\mathbb{E}}||J({\phi_t},{\lambda^{*}}({\phi_t}))- J({\phi_t},{\lambda}^{K}_{t})||}_{A^{''}} + \underbrace{{\mathbb{E}}||{\nabla_{\phi}}J({\phi_t},{\lambda}^{K}_{t})-  {\nabla_{\phi}}{J}({{\phi_t},{\lambda}^{K}_{t}},B)||^{2}}_{B^{''}}. \label{lem7_1}
\end{eqnarray}

We get Equation \eqref{lem7_0} form Equation \eqref{lem7_b} by using Assumption \ref{assump_1}. The first term $A^{''}$ is upper the same way starting from Equation \eqref{lem6_6} as in Lemma \ref{lemma_6} to give

\begin{eqnarray}
  {\mathbb{E}}||J({\phi_t},\lambda^{*}({\phi_t}))-J({\phi_t},\lambda^{K}_{t})||  &\le& 
  \tilde{\mathcal{O}}\left(\frac{1}{B}\right) + \tilde{\mathcal{O}}\left({\exp}^{-K}\right)   \label{lem7_2} 
\end{eqnarray}

$B^{''}$ is bounded in the same manner as $B_{k}^{'}$  in Lemma \ref{lemma_2} to give
\begin{eqnarray}
  {\mathbb{E}}||{\nabla_{\phi}}J({\phi_t},{\lambda}^{K}_{t})-  {\nabla_{\phi}}{J}({{\phi},{\lambda}^{K}_{t}(\phi)},B)|| &\le& 
  \tilde{\mathcal{O}}\left(\frac{1}{B}\right)  \label{lem7_3} 
\end{eqnarray}

Plugging Equation \eqref{lem7_2} and \eqref{lem7_3} into Equation \eqref{lem7_1} given us the required result.

\end{proof}
\begin{lemma} \label{lemma_8}
For a fixed ${\phi_t} \in \Theta$ and iteration $t$ of Algorithm \ref{algo_1} under Assumptions \ref{assump_1}-\ref{assump_5} and Assumptions \ref{assump_7}  we have 
\begin{eqnarray}
    {\mathbb{E}}||{\nabla_{\phi}}J({\phi_t},{\lambda_{\sigma}^{*}}({\phi_t}))-  {\nabla_{\phi}}{J}({{\phi_t},{\lambda^{'}}^{K}_{t}},B)||^{2} &\le&  
 \tilde{\mathcal{O}}\left(\frac{1}{{B}}\right) +
 \tilde{\mathcal{O}}\left({\exp}^{-K}\right)  
\end{eqnarray}
\end{lemma}
\begin{proof}
\begin{eqnarray}
       && {\mathbb{E}}||{\nabla_{\phi}}J({\phi_t},{\lambda_{\sigma}^{*}}({\phi_t}))-  {\nabla_{\phi}}{J}({{\phi},{\lambda^{'}}^{K}_{t}},B)||^{2} \nonumber\\
       &\le&  {\mathbb{E}}||{\nabla_{\phi}}J({\phi_t},{\lambda_{\sigma}^{*}}({\phi_t}))- {\nabla_{\phi}}J({\phi_t},{\lambda^{'}}^{K}_{t}) + {\nabla_{\phi}}J({\phi_t},{\lambda^{'}}^{K}_{t}) - {\nabla_{\phi}}{J}({{\phi_t},{\lambda^{'}}^{K}_{t}(\phi)},B)||^{2},\\
        &\le&  {\mathbb{E}}||{\nabla_{\phi}}J({\phi_t},{\lambda_{\sigma}^{*}}({\phi_t}))- {\nabla_{\phi}}J({\phi_t},{\lambda^{'}}^{K}_{t})||^{2} + ||{\nabla_{\phi}}J({\phi_t},{\lambda^{'}}^{K}_{t}) - {\nabla_{\phi}}{J}({{\phi_t},{\lambda^{'}}^{K}_{t}(\phi)},B)||^{2},\\
        &\le& L_{J}{\mathbb{E}}||({\lambda_{\sigma}^{*}({\phi_t})})- ({\lambda^{'}}^{K}_{t})||^{2} + {\mathbb{E}}||{\nabla_{\phi}}J({\phi_t},{\lambda^{'}}^{K}_{t})-  {\nabla_{\phi}}{J}({{\phi_t},{\lambda^{'}}^{K}_{t}(\phi)},B)||^{2}, \label{lem8_0}\\
        &\le&  \underbrace{L_{J}.L_{\sigma}{\mathbb{E}}||h_{\sigma}(\phi,{\lambda_{\sigma}^{*}}({\phi_t}))- h_{\sigma}({\phi_t},{\lambda^{'}}^{K}_{t})||}_{A^{''}} + \underbrace{{\mathbb{E}}||{\nabla_{\phi}}J(\phi,{\lambda^{'}}^{K}_{t})-  {\nabla_{\phi}}{J}({{\phi},{\lambda^{'}}^{K}_{t}},B)||^{2}}_{B^{''}}.\label{lem8_1}
\end{eqnarray}
We get Equation \eqref{lem8_1} from Equation \eqref{lem8_0} using Assumption \ref{assump_1}. Note that $B^{''}$ can be bounded same as $B_{k}^{'}$ in Lemma \ref{lemma_2}. Thus we have 
\begin{eqnarray}
     {\mathbb{E}}||{\nabla_{\phi}}J({\phi_t},{\lambda^{'}}^{k}_{t})-  {\nabla_{\phi}}{J}({{\phi_t},{\lambda^{'}}^{K}_{t}(\phi)},B)||^{2} &\le& \tilde{\mathcal{O}}\left(\frac{1}{{B}}\right)  \label{lem8_2}
\end{eqnarray}
For $A^{''}$ note that now the gradient descent is happening on the objective given by $h_{\sigma} = J({\lambda},{\phi}) -  {\sigma}G(\phi,\lambda)$. Applying the same logic as we did for $J(\phi,\lambda)$, from Assumption \ref{assump_1} and Lemma \ref{thm_supp} we get  
\begin{eqnarray}
 {\mathbb{E}}||h_{\sigma}({\phi_t},{\lambda_{\sigma}^{*}}({\phi_t}))- h_{\sigma}({\phi_t},{\lambda^{'}}^{k}_{t})||  &\le&   \tilde{\mathcal{O}}\left({\exp}^{-K}\right)  + \delta(B) \label{lem8_3}
\end{eqnarray}
where $\delta(B)$ is such that  $\mathbb{E}||\nabla_{\lambda}h_{\sigma}({\phi_t},{\lambda}) - \nabla_{\lambda}\hat{h}_{\sigma}({\phi_t},{\lambda})||^{2} \le \delta(B)$

Now, consider the term $\mathbb{E}||{\nabla}h_{\sigma}({\phi_t},{\lambda}) - {\nabla}\hat{h}_{\sigma}({\phi_t},{\lambda})||^{2}$, it can be decomposed as follows
\begin{eqnarray}
&&\mathbb{E}||{\nabla}h_{\sigma}({\phi_t},{\lambda}) - {\nabla}\hat{h}_{\sigma}({\phi_t},{\lambda})||^{2} \nonumber\\
&=& \mathbb{E}||{\nabla}_{\lambda}J({{\phi_t}},{\lambda}) + {\sigma}{\nabla}_{\lambda}G({{\phi_t}},{\lambda}) - {\nabla}_{\lambda}{J}({\phi_t},\lambda,B) - {\sigma}{\nabla}_{\lambda}G({{\phi_t}},{\lambda},B)||^{2}, \label{lem8_4}\\
 &\le&  \underbrace{\mathbb{E}||{\nabla}_{\lambda}J({{\phi_t}},{\lambda})-{\nabla}_{\lambda}{J}({\phi_t},{\lambda},B)||^{2}}_{A^{'''}} + {\sigma}\underbrace{\mathbb{E}||{\nabla}_{\lambda}G({{\phi_t}},{\lambda})-{\nabla}_{\lambda}{G}({{\phi_t}},{\lambda},n)||^{2}}_{B^{'''}}. \label{lem8_5}
\end{eqnarray}

Note that both $A^{'''}$ and $B^{'''}$ can be bounded same as $B_{k}^{'}$ in Lemma \ref{lemma_2}. thus we have 

\begin{eqnarray}
A^{'''} \le \tilde{\mathcal{O}}\left(\frac{1}{{B}}\right) \label{lem8_6}\\
B^{'''} \le \tilde{\mathcal{O}}\left(\frac{1}{{B}}\right) \label{lem8_7}
\end{eqnarray}

Thus we have $\delta(B) = \tilde{\mathcal{O}}\left(\frac{1}{B}\right)$. Which gives us
\begin{eqnarray}
 {\mathbb{E}}||h_{\sigma}(,{\phi_t},{\lambda_{\sigma}^{*}}({\phi_t})- h_{\sigma}({\phi_t},{\lambda^{'}}^{k}_{t})||  &\le&   \tilde{\mathcal{O}}\left({\exp}^{-K}\right)  + \tilde{\mathcal{O}}\left(\frac{1}{B}\right) \label{lem8_8}
\end{eqnarray}

Plugging Equation \eqref{lem8_8} and \eqref{lem8_2} into Equation \eqref{lem8_1}  gives us the required result.
\end{proof}



\section{Experiments}
\label{appendix:experiments}

\subsection{Setup}
The upper objective function to evaluate the reward is defined as follows
\begin{eqnarray}
G(\lambda,\phi) = -\mathbb{E}_{y,\tau_{0},\tau_{1} \sim \rho_{H}(\lambda)}(y{\cdot}P_{\phi}(\tau_0 > \tau_1) + (1-y){\cdot}(1-P_{\phi}(\tau_0 > \tau_1)))    \label{ps_4}
\end{eqnarray}
Where $\rho_{H}(\lambda)$ is the distribution of a trajectory of length $H$ by following the policy $\lambda$ and $y$ is the preference  which is $1$ if trajectory $1$ is preferred and $0$ if Trajectory $0$ is preferred which is drawn from some unknown distribution $\rho$. Also, $P_{\phi}(\tau_0 > \tau_1)$ is defined as
\begin{align}
    P_{\phi}(\tau_0 > \tau_1) = \frac{\exp{\sum_{h =0}^{H-1} r_{\phi}(s_h^0, a_h^0)}}{\exp{\sum_{h =0}^{H-1} r_{\phi}(s_h^0, a_h^0)} + \exp{\sum_{h =0}^{H-1} r_{\phi}(s_h^1, a_h^1)}}, \label{ps_5}
\end{align}

The  objective to be minimized is given in Equation \eqref{mdp_ps_9} as follows:
\[
\Phi_{\sigma}(\phi)
= \min_{\lambda} \left[
G(\phi,\lambda) + \frac{1}{\sigma} \big( J(\phi,\lambda^{\ast}(\phi)) - J(\phi,\lambda) \big)
\right],
\]
where \(\lambda^{\ast}(\phi) = \arg\max_{\lambda} J(\phi,\lambda)\) (noting the sign convention for the lower-level maximization of the return \(J\)).

To make this more implementable in an RL context, we reformulate the lower-level optimality using value functions. Let \(V(\phi,\lambda)\) denote the value function under policy \(\pi_{\lambda}\) (i.e., \(J(\phi,\lambda) = \mathbb{E}_{s\sim \nu,\, a\sim \pi_{\lambda}}[V(\phi,\lambda)]\), where \(\nu\) is the initial state distribution). The optimal lower-level policy should maximize the value function, and should therefore satisfy
\[
V(\phi,\lambda^{\ast}(\phi)) = V^{\ast}(\phi) = \max_{\lambda} V(\phi,\lambda).
\]
Substituting this into the penalty form yields:
\[
G(\phi,\lambda) + \frac{1}{\sigma}\big( V(\phi,\lambda^{\ast}(\phi)) - V(\phi,\lambda) \big)
= G(\phi,\lambda) + \frac{1}{\sigma}\big( V^{\ast}(\phi) - V(\phi,\lambda) \big).
\]

Directly minimizing the objective in Equation \ref{mdp_ps_9} is difficult in practice. Thus, for implementation, we make a practical approximation by dropping the \(V(\phi,\lambda)\) term (which is non-negative under the assumption of non-negative rewards, a common setup in discounted MDPs where \(V(\phi,\lambda) \ge 0\)). This provides an upper bound on the objective while simplifying computation:
\[
G(\phi,\lambda) + \frac{1}{\sigma} V^{\ast}(\phi).
\]
In the code, this manifests as the regularization term added to the upper-level loss \(G\), effectively encouraging the outer optimization (over \(\phi\)) to maximize the optimal value \(V^{\ast}\) scaled by \(1/\sigma\). This aligns with the bi-level structure by implicitly penalizing deviations from lower-level optimality without explicit inner-loop solving for \(\lambda^{\ast}\) at every step. We demonstrate improved performance over the PEBBLE~\cite{lee2020pebble} baseline in the two benchmarks using this approximation.  We leave the implementation of the full Algorithm \ref{algo_1} as well as obtaining a tighter upper bound on Equation \eqref{mdp_ps_9}  to future work.

\subsection{Implementation Details}

We evaluate the effectiveness of this method, which solves the simplified objective, on two distinct environments: the Walker locomotion task from the DeepMind Control Suite~\cite{tassa2018deepmind} and the Door Open manipulation task from Meta-world~\cite{mclean2025metaworld}. These environments are chosen as representative benchmarks for robotic locomotion and manipulation, respectively, and both present the challenge of learning from limited, preference-based feedback rather than direct access to ground-truth rewards.

To demonstrate the efficacy of this approach, we compare against PEBBLE~\cite{lee2020pebble} baseline, which also uses preference-based feedback for solving complex tasks. Both PEBBLE as well as the proposed method utilize unsupervised exploration as proposed in PEBBLE~\cite{lee2020pebble}, with disagreement-based sampling for query selection, a standard approach in preference-based reinforcement learning~\cite{DBLP:conf/corl/MetcalfSMT23}. For the PEBBLE baseline, we employ the publicly released code from B-Pref~\cite{lee2021bpref}, maintaining identical hyper-parameters and network architectures, such as the number of layers, learning rate, and the frequency of supervised reward learning. Our method builds on the PEBBLE framework, leveraging its core components while introducing our core contributions. We provide each task with a fixed budget of human preference labels: 100 labels for the Walker task and 1,000 labels for the Door Open task. All experiments are conducted on a single machine with an NVIDIA RTX 1080 Ti GPU, and we report results averaged over multiple independent runs with different random seeds.

\subsection{Results}
The training curves in Figure~\ref{fig:reward_ablation} illustrate the performance improvement of this approach against PEBBLE on both the Walker and Door Open tasks. In the Walker environment, the agent is rewarded for moving forward, and in our setting, the agent receives only preference-based feedback. The proposed method demonstrates improvements over the PEBBLE baseline, achieving higher average velocities and more stable learning trajectories with few preference labels. On the Door Open manipulation task, this approach similarly outperforms the baseline, successfully opening the door more consistently and efficiently. 

These results highlight the effectiveness of this method in improving feedback efficiency and task performance, even in settings with limited preference-feedback. It is to be noted that this approach improves over the PEBBLE baseline without the need for second-order terms, unlike \cite{chakraborty2024parl}. Other bi-level works such as \cite{shen2024principled} do not demonstrate improvement over state-of-the-art bi-level algorithms.  Overall, these experiments validate the advantages of this proposed approach in both locomotion and manipulation scenarios, underscoring its potential for real-world robotic applications. The code can be found at \url{https://github.com/MuditGaur/Neurips_2025_Bilevel_RL}.

\begin{figure}[H]
\centering
\includegraphics[width=0.49\linewidth]{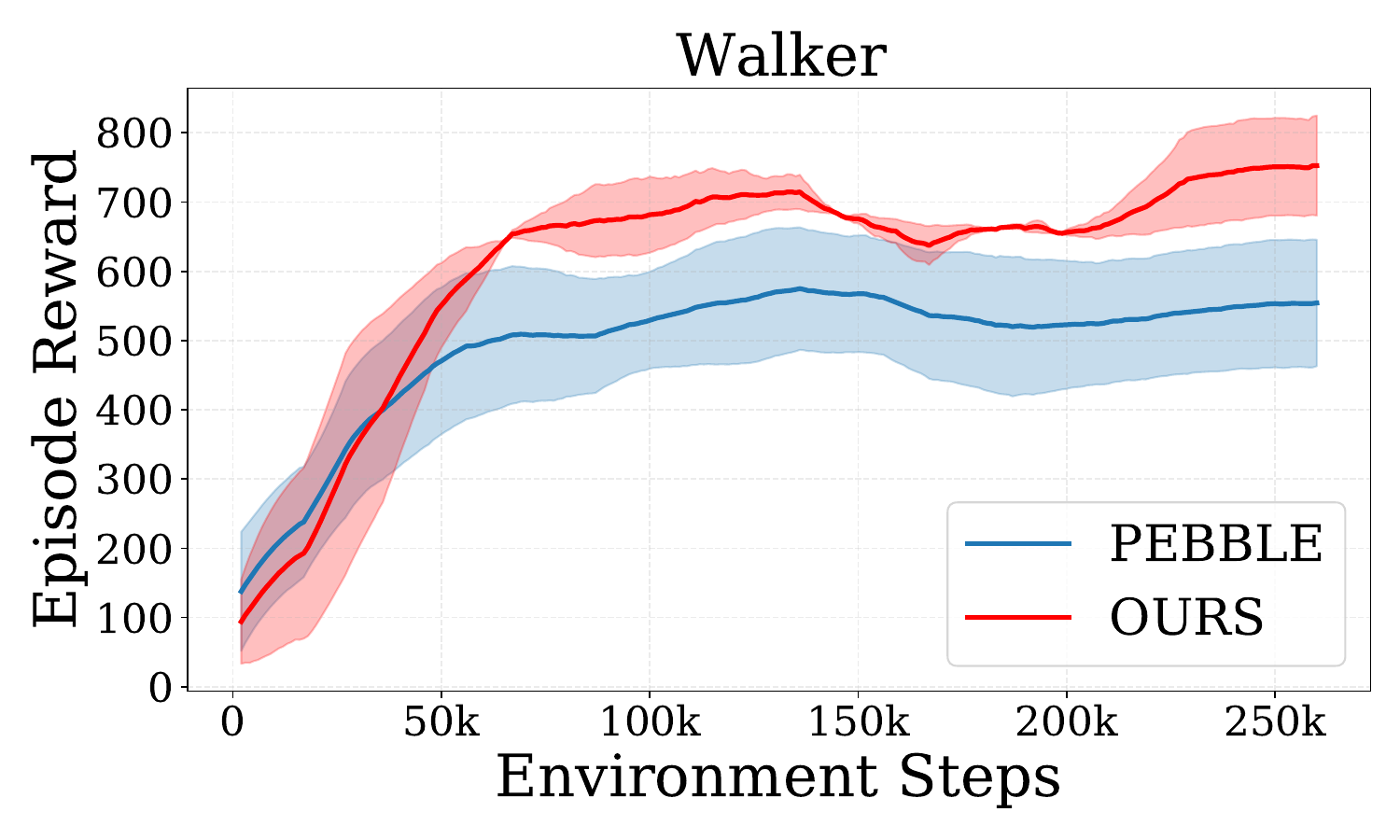}
\includegraphics[width=0.49\linewidth]{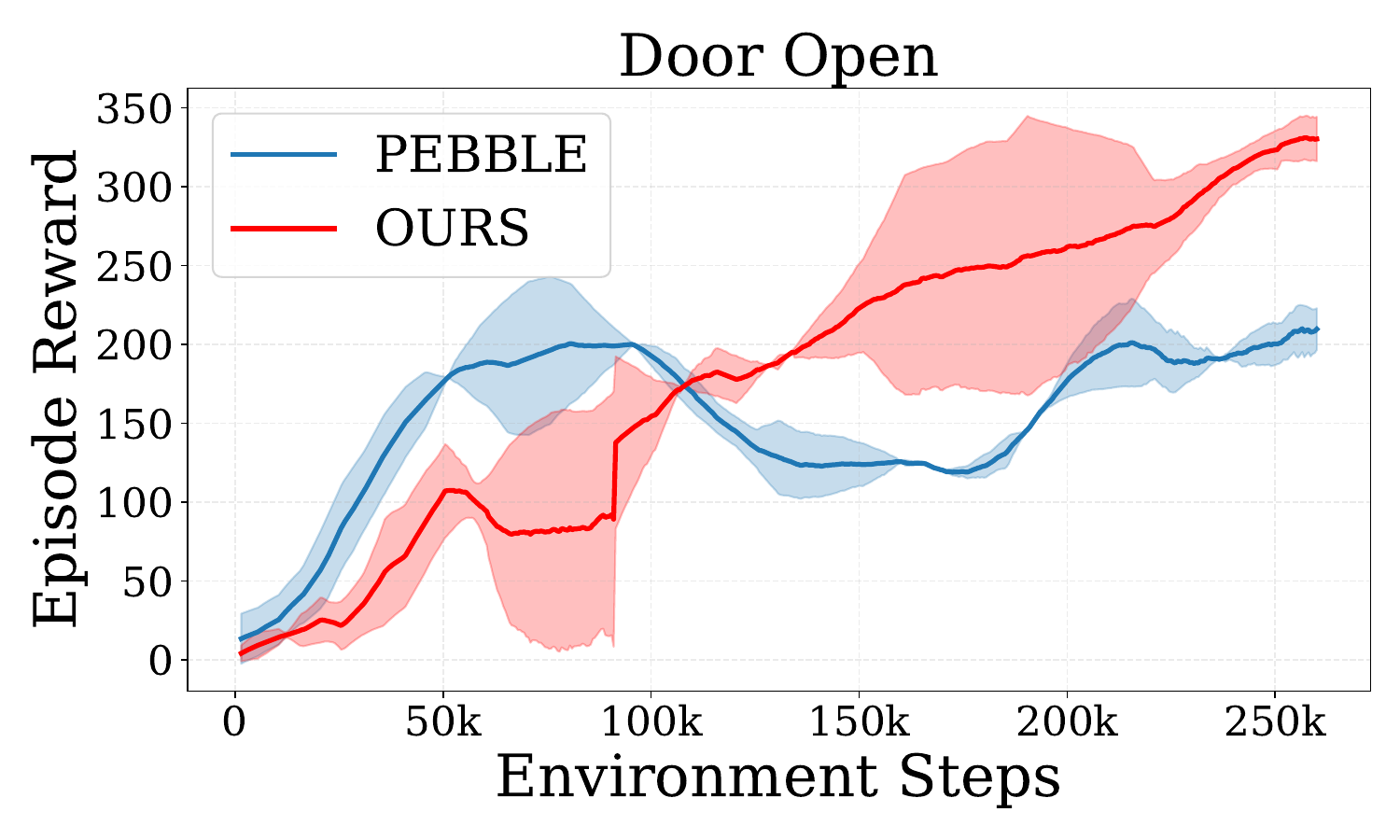}
\caption{Training curves on Walker locomotion task (left) from the DeepMind Control Suite~\cite{tassa2018deepmind} and the Door Open manipulation task (right) from Meta-world~\cite{mclean2025metaworld}. The solid line and shaded regions respectively, denote mean and standard deviation of the success rate, across multiple seeds. Blue curve: PEBBLE, Red curve: OURS.}
\label{fig:reward_ablation}
\end{figure}

\end{document}